\documentclass[runningheads]{llncs}

\usepackage{eccv}

\usepackage{eccvabbrv}

\usepackage{graphicx}
\usepackage{booktabs}
\usepackage{multirow}
\usepackage{tabularx}
\usepackage{wrapfig}
\usepackage{subcaption}
\usepackage{colortbl}  
\definecolor{lightgray}{RGB}{230,250,255}
\definecolor{palegray}{RGB}{245,253,255}

\usepackage[accsupp]{axessibility}  % Improves PDF readability for those with disabilities.

\usepackage{hyperref}

\usepackage{orcidlink}

\DeclareMathOperator*{\argmax}{arg\,max}
\DeclareMathOperator*{\argmin}{arg\,min}
\newcommand*\diff{\mathop{}\!\mathrm{d}}
\providecommand{\citep}[1]{\cite{#1}}
\providecommand{\citet}[1]{\cite{#1}}
\definecolor{DodgerBlue3}{RGB}{24,116,205}
\definecolor{darkgreen}{RGB}{0,128,0}

\newcommand{\edit}[1]{\textcolor{black}{#1}}

\begin{document}

\title{Diffusion Integrated Gradients: Controllable Path Generation for Flexible Feature Attribution}
% \title{Integrated Gradients with Diffusion\\for Flexible Input Attribution}
\titlerunning{Diffusion Integrated Gradients}

% % TODO FINAL: Replace with your author list.
% % Include the authors' ORCID for the camera-ready version, if at all possible.

\author{
Soyeon Kim\inst{1,3}\orcidlink{0009-0001-5037-0902} \and
Kyowoon Lee\inst{2}\orcidlink{0009-0001-3600-8237}\textsuperscript{\dag} \and
Jaesik Choi\inst{1,3}\orcidlink{0000-0002-4663-3263}\textsuperscript{\dag}
}

% % TODO FINAL: Replace with an abbreviated list of authors.
\authorrunning{S.~Kim et al.}

\institute{%
\makebox[\textwidth][c]{%
\begin{tabular}{@{}c@{\hspace{2em}}c@{\hspace{2em}}c@{}}
\textsuperscript{1}KAIST, Korea &
\textsuperscript{2}Ajou University, Korea &
\textsuperscript{3}INEEJI Corp., Korea
\end{tabular}%
}\\[0.2em]
\texttt{kyowoon.lee1924@gmail.com, \{soyeon.k,jaesik.choi\}@kaist.ac.kr}\\[0.5em]
\textsuperscript{\dag}Co-corresponding authors.
}

\let\oldaddcontentsline\addcontentsline
\renewcommand{\addcontentsline}[3]{}

\maketitle

\let\addcontentsline\oldaddcontentsline

\begin{abstract}
Path-based attribution methods such as Integrated Gradients (IG) are widely adopted for their strong axiomatic properties and effectiveness in attributing model predictions to input features by integrating gradients along a path from a baseline to the input. However, the choice of the attribution path largely affects the quality of explanations, and existing approaches rely on fixed or hand-crafted paths that often produce noisy or distorted attributions. To address this limitation, we propose \textit{Diffusion Integrated Gradients} (DiffIG), a novel method that reformulates path generation as a conditional generative modeling problem. DiffIG first trains a diffusion model to learn a distribution over paths generated from a \edit{Stick-Breaking Process}, then employs guided sampling to embed user guidance during the sampling procedure. We demonstrate that DiffIG quantitatively matches or outperforms existing path-based methods, achieving perceptually aligned explanations. This work introduces a new generative perspective for flexible, inference-time controllable Explainable Artificial Intelligence (XAI) methods. 
\end{abstract}

\begin{figure*}[t!]
    \centering
    \includegraphics[width=\linewidth]{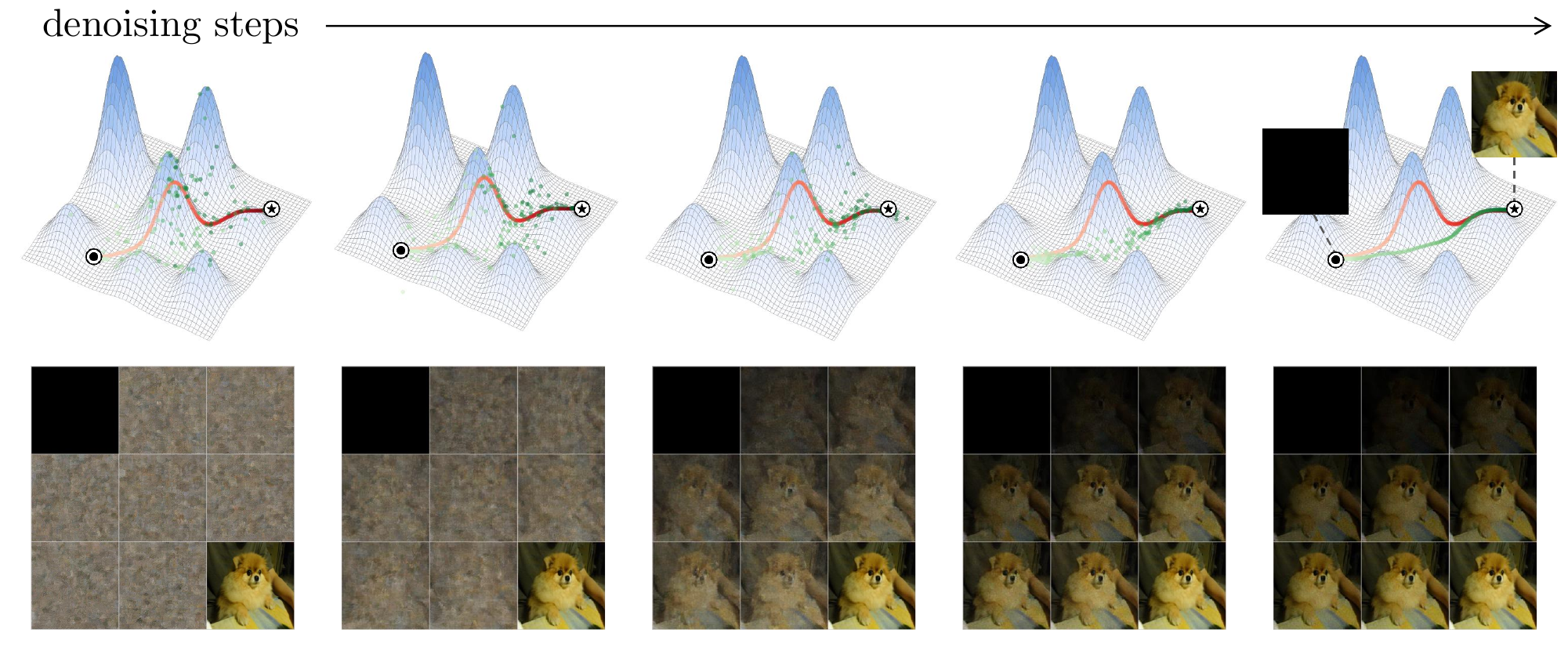}
    \caption{\textbf{Illustration of the path generation process of DiffIG.} 
    \textbf{(Top)} The diffusion model generates an attribution path from random noise through an iterative denoising process. 
    The {\color{DodgerBlue3}blue meshed field} refers to the logit surface of the model. 
    The {\color{red}red path} represents the IG straight-line path, which is susceptible to noisy gradient accumulation and often traverses irregular regions of the model surface. 
    In contrast, the {\color{darkgreen}green path} denotes the DiffIG path, which leverages guided sampling to avoid regions that induce noisy gradients and produces more faithful attributions through controllable guidance (see Sec.~\ref{sec:approach}). 
    \textbf{(Bottom)} Visualization in the image space, where DiffIG transforms a noisy initial trajectory into a non-linear path connecting the baseline point \protect{\raisebox{-.05cm}{\includegraphics[height=.35cm]{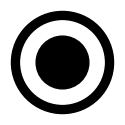}}} to the input point \protect{\raisebox{-.05cm}{\includegraphics[height=.35cm]{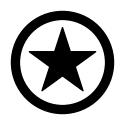}}}.
}
    \vspace{-0.2in}
    \label{fig:denoising_process}
\end{figure*}

\section{Introduction}
\label{sec:intro}

% Deep Neural Networks (DNNs) have achieved remarkable success across domains, but their black-box nature hinders deployment in high-stakes applications such as 
Deep Neural Networks (DNNs) have demonstrated exceptional performance in a wide range of domains; however, their opaque, black-box nature limits their use in high-stakes settings such as 
medical diagnosis, autonomous driving, \edit{or weather forecasting~\cite{kim2023explainable}}, where trust, debugging, safety, and fairness are essential~\cite{chen2021deep,lee2023adaptive}. \edit{Explainable Artificial Intelligence} (XAI), particularly input attribution, addresses this by assigning each input feature an importance score reflecting its contribution to the model output.
Consequently, a wide array of methods, such as Saliency Maps~\cite{simonyan2014saliency}, SHAP~\cite{lundberg2017unified}, and LRP~\citep{montavon2019layer} have been proposed. 
% However, the reliability of early attribution methods remains questionable, as they use heuristics and have failed to pass fundamental sanity checks~\cite{adebayo2018sanity}, revealing their potential to produce misleading or unfaithful explanations~\citep{kindermans2019reliability}.
However, the trustworthiness of early attribution methods remains uncertain, since they often rely on heuristics and have been shown to fail basic sanity checks~\cite{adebayo2018sanity}. This raises concerns that such methods may generate explanations that are misleading or not faithful to the underlying model behavior~\citep{kindermans2019reliability}.

To provide more principled explanations, Integrated Gradients (IG)~\citep{sundararajan2017axiomatic} emerged as a prominent solution. 
% IG is grounded in the Aumann–Shapley value framework and path-based formulations~\citep{aumann1974values, friedman2004paths}, 
IG is theoretically rooted in the Aumann–Shapley value framework and path-based attribution formulations~\citep{aumann1974values, friedman2004paths}, 
which provide strong theoretical guarantees. In particular, IG uniquely satisfies critical axioms such as \textit{Efficiency} or \textit{Completeness}, ensuring 
% that feature attributions sum to the difference between the model output at the input and at the baseline. IG computes attributions by integrating gradients along a straight path from a neutral baseline (e.g., a black image) to the target input. Due to this axiomatic rigor, IG has been widely adopted as a standard for feature attribution in complex models.
that the resulting feature attributions add up to the difference between the model’s prediction for the input and its prediction for the baseline. IG assigns importance scores by integrating gradients along a straight-line path from a neutral baseline, such as a black image, to the target input. Owing to this strong axiomatic foundation, IG has become widely used as a standard feature attribution method for complex models.

Despite these strengths, IG critically depends on the choice of integration path. The standard straight-line path is theoretically justified under the Aumann–Shapley framework~\citep{aumann1974values}, where it is uniquely consistent with the axioms and probabilistically optimal for functions whose gradient fields are \textit{conservative}, i.e., the path integral is path-independent and any baseline-to-input path yields the same attribution. This assumption breaks down for modern rectified DNNs: their piecewise-linear nature and discontinuous gradients at linear-region boundaries make the gradient field \textit{highly discontinuous}, so the attribution integral becomes \textit{path-dependent} and the straight-line path loses its optimality. Worse, this path can traverse complex regions of the function space, aggregating spurious feature contributions and misleading attributions~\citep{srinivas2019full, jeon2022distilled, zaher2024manifold, yang2023local}.

We propose \textbf{Diff}usion \textbf{I}ntegrated \textbf{G}radients (\textbf{DiffIG}), a path-based attribution method that learns integration paths via generative modeling. We train a diffusion model on the distribution of non-linear paths from a Stick-Breaking Process, capturing diverse trajectories of varying form and complexity. DiffIG further leverages guided sampling to embed user-defined priors at inference, steering path generation toward properties such as faithfulness and complexity. This transforms path-based attribution from a fixed procedure into a flexible, controllable generative process.

Our main contributions are as follows: 
\textbf{(1)} We reformulate attribution path generation as a conditional generative modeling problem.
\textbf{(2)} We propose \textbf{Diff}usion \textbf{I}ntegrated \textbf{G}radients (DiffIG), a novel method that enables flexible and inference-time controllable feature attribution.
\textbf{(3)} We demonstrate through quantitative and qualitative experiments that DiffIG \edit{matches or outperforms} existing path-based methods, yielding more faithful and perceptually aligned explanations.

\section{Background}
\subsection{Path-based Attribution Methods}
\label{subsec:path}

For an input $x \in \mathbb{R}^n$ and a differentiable model $f: \mathbb{R}^n \to \mathbb{R}$, path-based attribution methods accumulate gradients along a path $\gamma(t)$ for $t \in [0, 1]$. This path connects a baseline $x'$ to the input $x$ (i.e., $\gamma(0) = x'$ and $\gamma(1) = x$). In vision models, the baseline $x'$ is typically chosen as either a black or a white image. The attribution $\mathcal{A}_i$ for the $i$-th feature is defined as~\citep{sundararajan2017axiomatic}:
\begin{align}
    \mathcal{A}_i(\gamma) = \int_{0}^{1} \frac{\partial f(\gamma(t))}{\partial \gamma_i(t)}  \frac{\partial \gamma_i(t)}{\partial t} \diff t ,
    \label{eq:path_general}
\end{align}
where the integral consists of two key terms: $\frac{\partial f(\gamma(t))}{\partial \gamma_i(t)}$, the model gradient at $\gamma(t)$, and $\frac{\partial \gamma_i(t)}{\partial t}$, the path direction.

\paragraph{\textbf{Integrated Gradients (IG).}}

When the path is the straight line $\gamma(t) = x' + t (x - x')$, Equation~\eqref{eq:path_general} recovers the canonical \textit{Integrated Gradients} (IG)~\citep{sundararajan2017axiomatic}. For this path, the direction term $\frac{\partial \gamma_i(t)}{\partial t}$ simplifies to $(x_i - x'_i)$, yielding:
\begin{align}
    \mathcal{A}_i(\gamma) = (x_i - x'_i) \int_{0}^{1} \frac{\partial f(x' + t (x - x'))}{\edit{\partial x_i}} \, \diff t.
    \label{eq:ig}
\end{align}

\paragraph{\textbf{Path-Dependency Problem in Rectified DNNs.}}
\label{sec:background_path_dependency}

The theoretical elegance of IG relies on the assumption that the gradient field of $f$ is \textit{conservative}. In this case, the path integral is \edit{path-independent}, and any $\gamma$ connecting $x'$ and $x$ yields a unique attribution. However, this assumption does not hold for modern rectified DNNs.
Due to their piecewise-linear structure~\cite{chu2018exact}, the input space is partitioned into numerous distinct decision regions. Consequently, the gradient field of such models is highly discontinuous~\citep{balduzzi2017shattered}, making the resulting feature-wise attribution integral highly sensitive to the chosen path and susceptible to numerical errors~\cite{kapishnikov2021guided}.

In practice, the standard straight-line path remains a convenient yet naive heuristic that ignores the complex geometry of the underlying decision surface of the model. As a result, it often traverses discontinuities at region boundaries, accumulating noisy or spurious gradients. 
This can lead to misleading feature attribution allocation, resulting in attributions that are unfaithful to the true decision-making process of the model.
Consequently, the key challenge for improving path-based attribution lies in discovering an adaptive and model-aware integration path.

\subsection{Diffusion Probabilistic Models}

Diffusion probabilistic models~\citep{sohl2015deep,ho2020denoising} are generative models that learn to reverse a fixed noising process.
The learned reverse process $p_\theta(\gamma^{\tau-1} \mid \gamma^{\tau})$ inverts a predefined forward process $q(\gamma^{\tau} \mid \gamma^{\tau-1})$ that gradually adds noise to data until it becomes a standard Gaussian prior. The model learns a data distribution represented as a Markov chain with Gaussian transitions:
\begin{align}
p_\theta(\gamma^{0}) 
&= \int p(\gamma^{M}) 
\prod_{\tau=1}^{M} p_\theta(\gamma^{\tau-1} \mid \gamma^{\tau}) 
\, \edit{\diff} \gamma^{1:M}.
\end{align}
Here, $\gamma^{0}$ represents the clean data, and $p(\gamma^{M})$ is the standard Gaussian prior after $M$ noising steps. 
The training objective is to find parameters $\theta$ that minimize a variational bound on the negative log-likelihood of the reverse process: $\theta^* = \argmin_{\theta} -\mathbb{E}_{\gamma^{0}}[\log p_\theta(\gamma^{0})]$. In practice, the reverse transition $p_\theta(\gamma^{\tau-1} \mid \gamma^{\tau})$ is often modeled as a Gaussian distribution with learnable means:
\begin{align}
\label{eq:reverse_process}
p_\theta(\gamma^{\tau-1} \mid \gamma^{\tau}) 
&= \mathcal{N}(\gamma^{\tau-1} \mid \mu_\theta(\gamma^{\tau}), \Sigma^{\tau}),
\end{align}
while the forward process $q(\gamma^{\tau} \mid \gamma^{\tau-1})$ is fixed and is typically specified by a predefined variance schedule.

\section{Diffusion Integrated Gradients}
\label{sec:approach}

We first formulate attribution path generation as conditional generative modeling, then introduce \textbf{Diffusion Integrated Gradients (DiffIG)}, a path-based attribution method that learns integration paths through generative modeling and guided sampling. As highlighted in Figure~\ref{fig:method_comparison_table}, DiffIG offers distinct advantages over existing path-based methods in adaptive path generation, global optimization, and controllability, and Figure~\ref{fig:denoising_process} overviews its path generation process.

\subsection{Path Generation with Diffusion}

As discussed in Sec.~\ref{subsec:path}, the quality of the resulting attributions heavily depends on the chosen path. In particular, methods such as Integrated Gradients (IG) employ a model-agnostic linear path, making their attributions highly sensitive to the underlying decision surface of the model. To mitigate the accumulation of noise and unfaithful attributions caused by such paths, path-based attribution methods aim to identify a path $\gamma$ that connects a baseline $x'$ to an input $x$. Formally, this can be formulated as finding an optimal path $\gamma^{*}$ that maximizes an objective function $\mathcal{J}$:
\begin{align}
\label{eq:path_objective}
    \gamma^{*} = \argmax_{\gamma \in \Gamma} \, \mathcal{J}(\gamma),
\end{align}
where $\Gamma$ denotes the set of all possible paths, $\gamma$ represents a path, and $\mathcal{J}$ indicates the corresponding objective value of that path.

For example, Guided IG (GIG) \cite{kapishnikov2021guided} seeks a path that maximizes the objective $\mathcal{J}(\gamma)$, defined as the negative of a cost combining path instability $\ell_{\text{noise}}$ and deviation from the straight-line path $\ell_{\text{distance}}$:
\begin{align}
    \mathcal{J}(\gamma) &= -(\ell_{\text{noise}} + \lambda \ell_{\text{distance}}) \nonumber \\
    \ell_{\text{noise}} &= \sum_{i=1}^{n} \int_{t=0}^1\left|\frac{\partial f(\gamma(t))}{\partial \gamma_i(t)}\frac{\partial \gamma_i(t)}{\partial t}\right| \diff t \nonumber \\
    \ell_{\text{distance}} &= \int_{t=0}^1\left|\left|\gamma(t)-\gamma_{\text{IG}}(t)\right|\right|\diff t \nonumber,
\end{align}
where $\gamma_{\text{IG}}(t)$ is the straight-line path, $\ell_{\text{noise}}$ quantifies the accumulated gradient magnitude on irrelevant features, and $\ell_{\text{distance}}$ measures the deviation from the linear path. However, optimizing such a complex objective over a continuous function space, particularly under the complex topology of the underlying decision surface of the model, is generally intractable and necessitates heuristic or greedy approximations. Consequently, GIG relies on a \textit{greedy}, iterative approximation that performs heuristic local updates, preventing it from identifying a globally consistent integration path.

Instead of explicitly minimizing an objective as in Eq.~\eqref{eq:path_objective}, we draw on offline reinforcement learning, where sequential decision-making is cast as generating desired trajectories conditioned on their outcomes~\citep{janner2022planning, ajay2023is}. Following this view, we reformulate finding an optimal path as conditional generative modeling:
\begin{align}
    \theta^{*} = \argmax_{\theta} \, \mathbb{E}_{\gamma \sim \mathcal{D}} \big[ \log p_{\theta}(\gamma|y=\mathcal{J}(\gamma) ) \big],
\end{align}
where $\mathcal{D}$ represents the dataset of pre-collected paths, and the conditional label $y$ denotes the score $\mathcal{J}(\gamma)$ associated with a path $\gamma$ from this dataset. By learning this conditional data distribution $p_{\theta}$, we can later generate new paths $\gamma$ that possess high $\mathcal{J}(\gamma)$ values.

Unlike unconstrained image synthesis, our objective is path generation between fixed endpoints.
To ensure that the generated path $\gamma(t)$ indeed connects the baseline and the input, we explicitly enforce the boundary conditions $\gamma(0) = x'$ and $\gamma(1) = x$ at every denoising step during sampling. 
This guarantees that the diffusion model generates only the valid intermediate attribution trajectories that start from the baseline and terminate at the input.

\setlength{\intextsep}{0pt}
\begin{wrapfigure}{l}{0.55\linewidth}
    \centering
    \includegraphics[width=\linewidth]{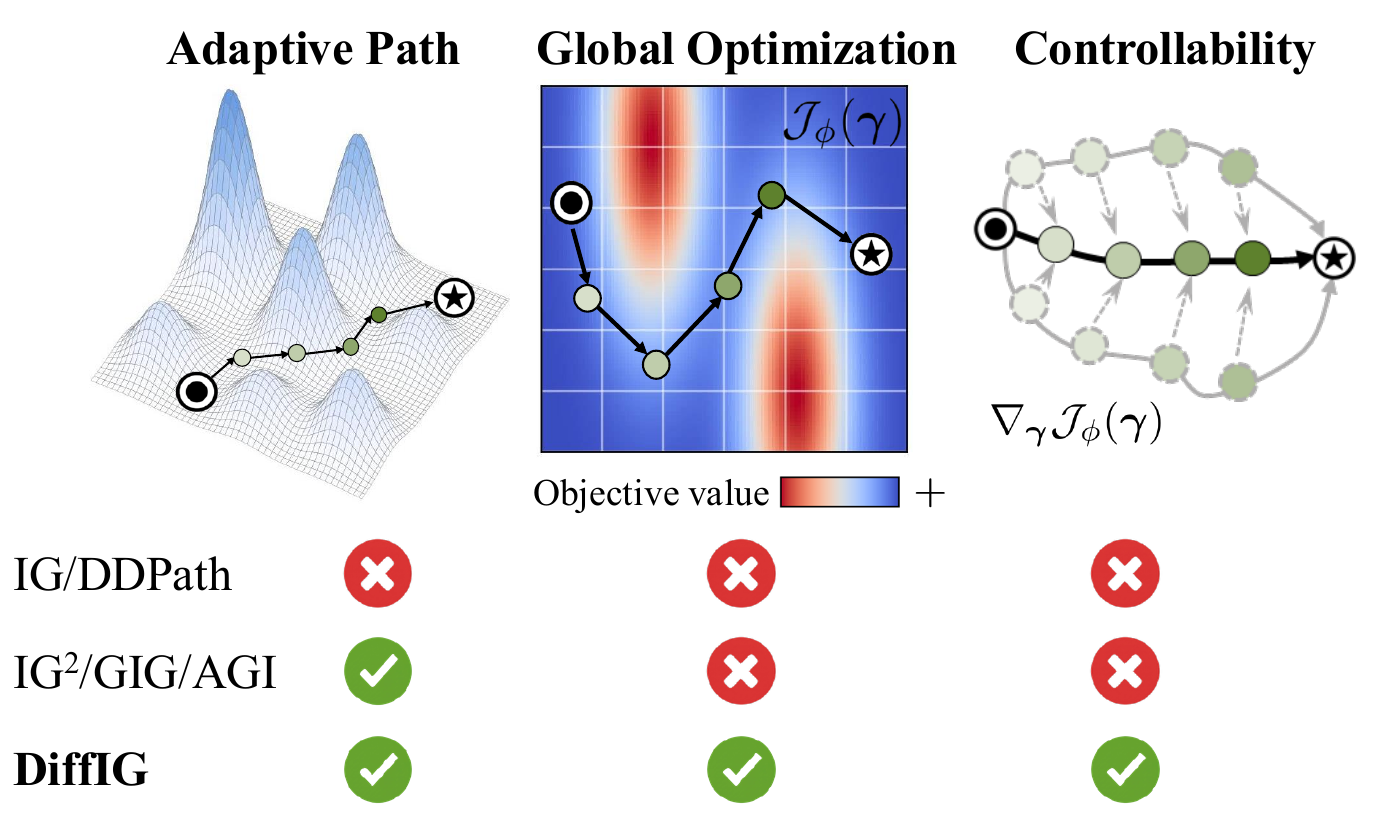}
    \caption{\textbf{DiffIG capabilities.} 
    \textbf{Adaptive Path} indicates whether a method can adapt its integration path to the complex decision surface of the model.
    \textbf{Global Optimization} refers to the ability to find a globally optimal path, rather than relying on greedy or local search heuristics.
    \textbf{Controllability} denotes the ability to flexibly guide path generation at inference time toward desired properties.}
    \label{fig:method_comparison_table}
\end{wrapfigure}

\paragraph{\textbf{Training.}}
We use diffusion probabilistic models \cite{sohl2015deep,ho2020denoising} to learn the conditional distribution $p_{\theta}(\gamma|y=\mathcal{J}(\gamma))$. However, instead of training this conditional distribution directly, we adopt an unconditional diffusion model with classifier guidance~\cite{dhariwal2021diffusion}. This approach involves training two separate models.
First, we train an unconditional diffusion model to learn the unconditional distribution of paths $p_{\theta}(\gamma)$ from the dataset $\mathcal{D}$. 
% As part of this procedure, we train an $\boldsymbol{\epsilon}$-model, $\boldsymbol{\epsilon}_\theta$, to predict the source noise instead of the mean $\boldsymbol{\mu}_\theta$. 
As part of this process, we train an $\boldsymbol{\epsilon}$-model, $\boldsymbol{\epsilon}_\theta$, to estimate the source noise rather than predicting the mean $\boldsymbol{\mu}_\theta$. 
This enables a simplified objective~\cite{ho2020denoising}:
\begin{align}
\label{eq:loss}
\mathcal{L}(\theta):=\mathbb{E}_{\tau,\boldsymbol{\epsilon},{\gamma}^0}[\|\boldsymbol{\epsilon} - \boldsymbol{\epsilon}_\theta({\gamma}^\tau)\|_2^2],
\end{align}
where $\tau \in \{1, ..., M\}$ is the diffusion timestep, $\boldsymbol{\epsilon} \sim \mathcal{N}(\mathbf{0},\mathbf{I})$ is the target noise, and ${\gamma}^\tau$ is the path corrupted by the noise $\boldsymbol{\epsilon}$ from the noiseless path ${\gamma}^0$.
Second, to steer the diffusion model towards generating paths with high objective scores, we use energy-guided sampling. This samples from a perturbed distribution $\tilde{p}_\theta({\gamma}^0)$ that prioritizes paths with higher objective scores $\mathcal{J}({\gamma}^0)$, defined as:
\begin{align}
    \tilde{p}_\theta({\gamma}^0) \propto p_\theta({\gamma}^0) \exp{(\mathcal{J}({\gamma}^0))}.
\end{align}
To enable this guidance, we train a separate regression network, $\mathcal{J}_\phi$. This network is trained to predict the objective score $\mathcal{J}({\gamma}^0)$ from a noisily perturbed version ${\gamma}^\tau$ at any timestep $\tau$. This is accomplished through the following mean-square-error (MSE) objective:
\begin{equation}
\label{Eq:MSE_qt_objective}
    \min_{\phi} \mathbb{E}_{\tau,\boldsymbol{\epsilon},{\gamma}^0} \left[
    \|\mathcal{J}_\phi({{\gamma}}^\tau) - \mathcal{J}({{\gamma}}^0)\|_2^2
\right].
\end{equation}
During the sampling stage, classifier guidance \cite{dhariwal2021diffusion} is employed, incorporating the gradients of $\mathcal{J}_\phi$ into the reverse diffusion process. Specifically, the mean of the reverse transition is updated as follows:
\begin{equation}
\label{eq:guided_sampling}
\tilde{p}_{\theta}({\gamma}^{\tau-1}|{\gamma}^\tau)=\mathcal{N}({\gamma}^{\tau-1}|\boldsymbol{\mu}_{\theta}({\gamma}^\tau)  + \omega \boldsymbol{\Sigma}^\tau g, \boldsymbol{\Sigma}^\tau) ,
\end{equation}
where $g=\nabla_{{{\gamma}}}\mathcal{J}_\phi({{\gamma}})|_{{{\gamma}} = \boldsymbol{\mu}_{\theta}({\gamma}^\tau)}$ and $\omega$ is the guidance scale that controls the strength of the guidance.

\setlength{\intextsep}{0pt}
\begin{wrapfigure}{L}{0.63\linewidth}
    \centering
    \begin{subfigure}[b]{0.32\linewidth}
        \centering
        \includegraphics[width=\linewidth]{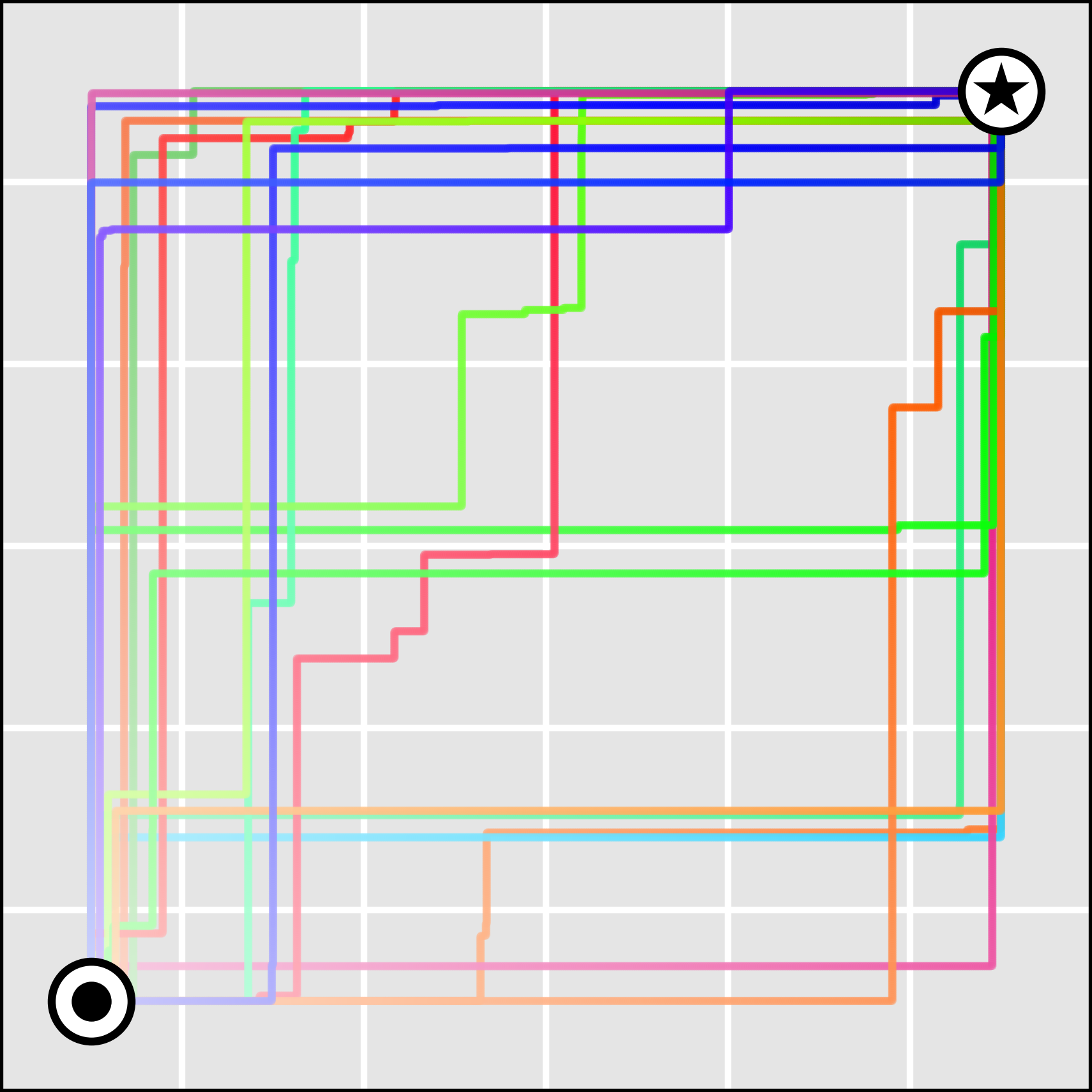}
        \caption{$\alpha = 1$}
        \label{subfig:sbp_alpha1}
    \end{subfigure}
    \hfill
    \begin{subfigure}[b]{0.32\linewidth}
        \centering
        \includegraphics[width=\linewidth]{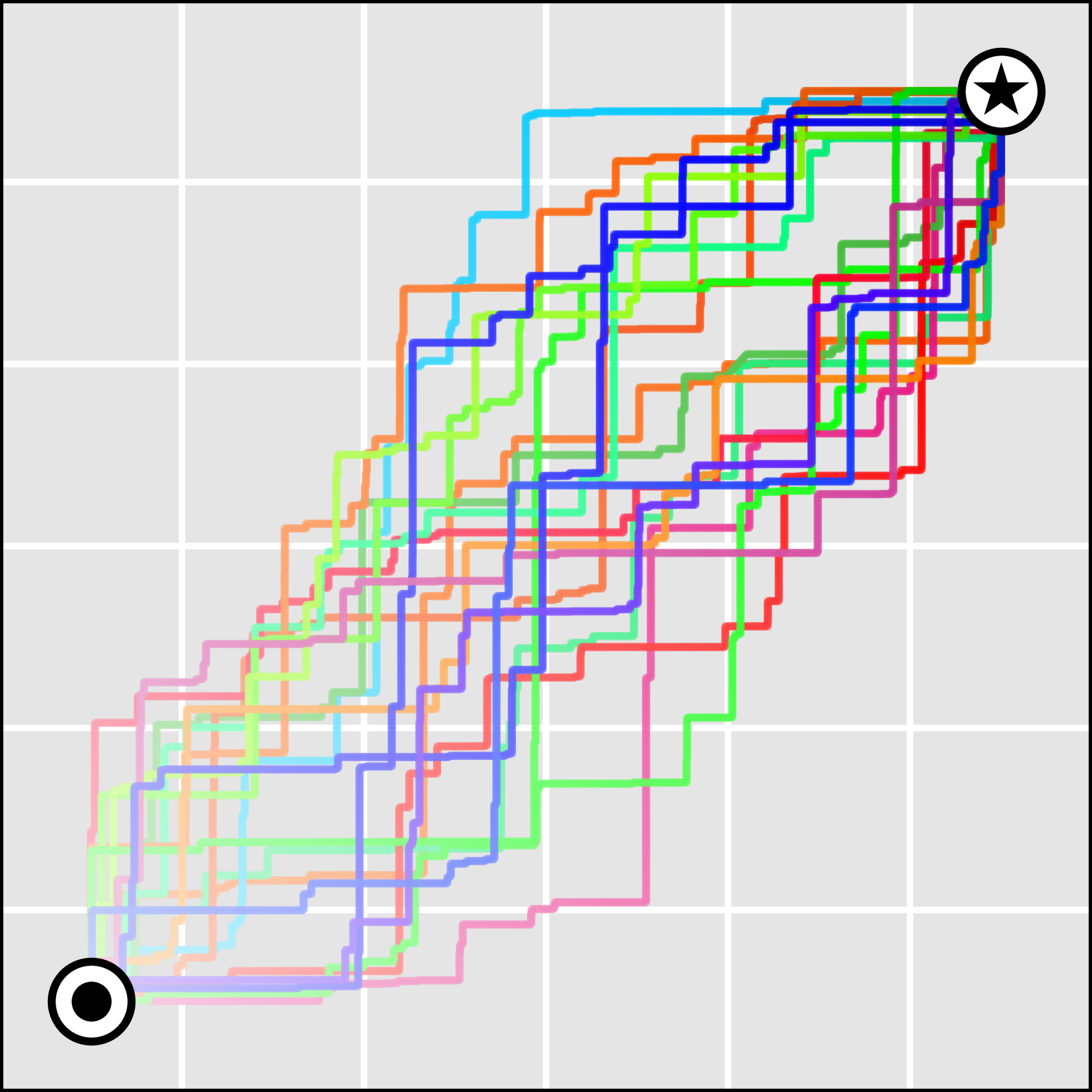}
        \caption{$\alpha = 10$}
        \label{subfig:sbp_alpha10}
    \end{subfigure}
    \hfill
    \begin{subfigure}[b]{0.32\linewidth}
        \centering
        \includegraphics[width=\linewidth]{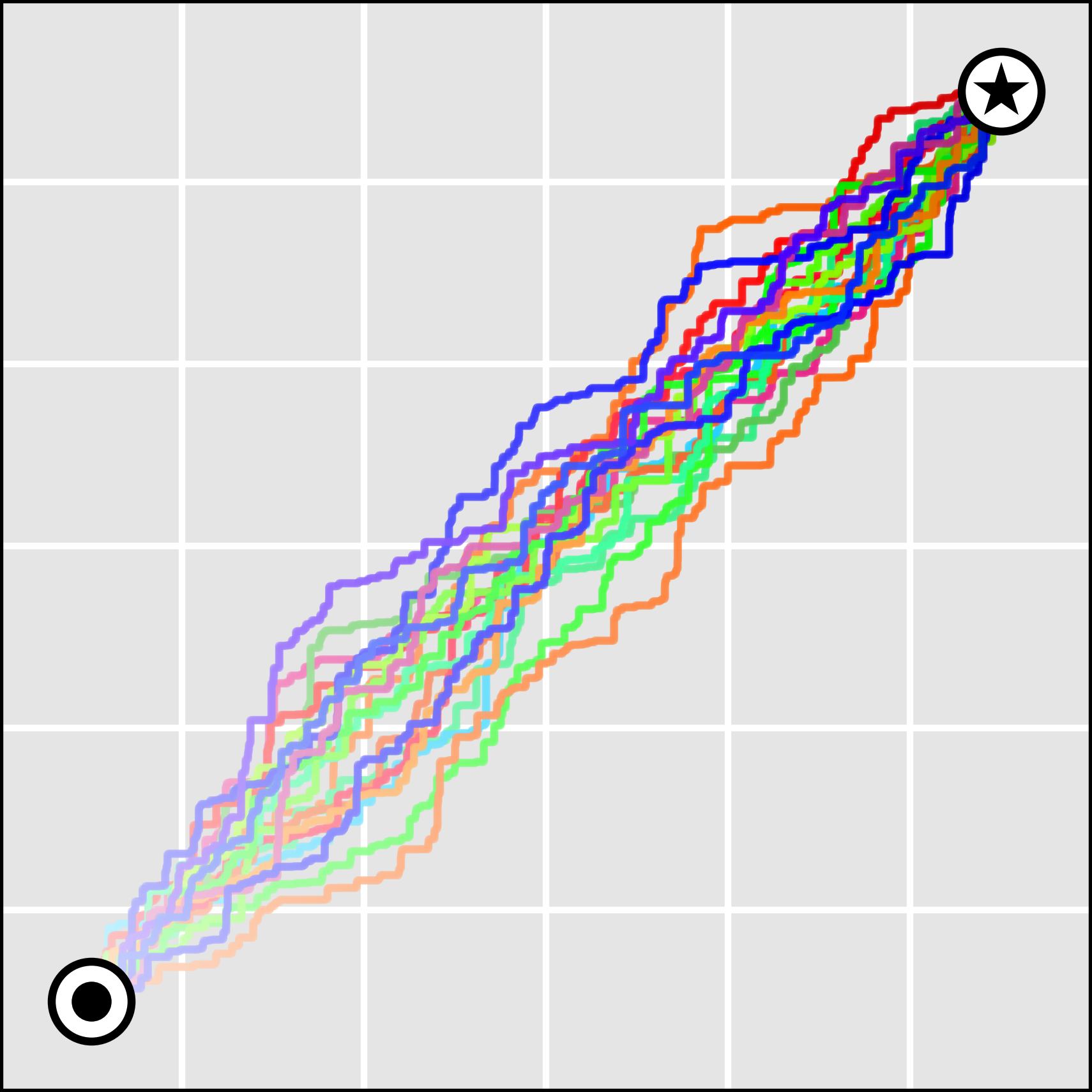}
        \caption{$\alpha = 100$}
        \label{subfig:sbp_alpha100}
    \end{subfigure}
    \caption{\textbf{2D visualization of stochastic paths generated by the Stick-Breaking Process (SBP).} 
    Each curve represents a sampled path from the baseline \protect{\raisebox{-.05cm}{\includegraphics[height=.35cm]{figures/mark_start_crop.png}}} to the input \protect{\raisebox{-.05cm}{\includegraphics[height=.35cm]{figures/mark_goal_crop.png}}}. The concentration parameter $\alpha$ controls the distribution of paths. A small $\alpha$ encourages high variance, producing diverse non-linear paths, while a large $\alpha$ constrains the paths to be smoother and closer to the linear path.}
    \label{fig:sbp_alpha}
\end{wrapfigure}

\subsection{Synthetic Path Data Generation}
Training the unconditional diffusion model and the guidance network requires a large and diverse dataset $\mathcal{D}$ of candidate paths. To this end, we adopt the Stick-Breaking Process (SBP) \cite{sethuraman1994constructive,jeon2023beyond} to generate a rich set of non-linear paths. 
Conceptually, SBP sequentially breaks a unit-length stick into random segments, where the resulting proportions define a random measure over the interval $[0,1]$. 

For each feature dimension $i$, we sample a measure $G_i$ via SBP:
\begin{align}
    G_i(t) &= \sum_{k=1}^{\infty} \pi_k \delta_{t_k}(t), \\
    \pi_k &= \beta_k \prod_{j=1}^{k-1} (1 - \beta_j), \\
    \beta_k &\sim \text{Beta}(1, \alpha), \quad t_k \sim H,
\end{align}
where $\delta_{t_k}(t)$ is a Dirac delta function centered at $t_k$, and $H$ is a base distribution. Following~\cite{jeon2023beyond}, we set the base distribution $H$ to the uniform distribution $U(0,1)$. 
The cumulative distribution function (CDF) $F_{G_i}(t)$ of $G_i$ is then used to define a non-linear interpolation between the baseline $x_i'$ and the input $x_i$:
\begin{align}
    \gamma_i(t) = x'_i + F_{G_i}(t)(x_i - x'_i).
\end{align}
This construction ensures that $\gamma_i(0) = x'_i$ and $\gamma_i(1) = x_i$. \edit{Since $F_{G_i}$ is a non-decreasing CDF with $F_{G_i}(0)=0$ and $F_{G_i}(1)=1$, $\gamma_i(t)$ transitions monotonically from $x'_i$ to $x_i$.} 
The concentration parameter $\alpha$ controls the diversity of generated paths. A large $\alpha$ produces paths close to the straight-line path used in standard IG, while a small $\alpha$ yields more diverse and non-linear trajectories. Figure~\ref{fig:sbp_alpha} illustrates how different $\alpha$ values lead to distinct path characteristics. To build $\mathcal{D}$, we generate paths across a wide range of $\alpha$, ensuring a rich diversity of trajectories ranging from nearly linear to highly complex.

\subsection{Guidance Objectives}
\label{subsec:objectives}

To guide the diffusion model toward generating high-quality attribution paths, we train the guidance network $\mathcal{J}_\phi$ to predict objective values that reflect the quality of each path. 
Specifically, we define two complementary objectives that capture distinct aspects of attribution quality: \textit{faithfulness} and \textit{complexity}. 

\paragraph{\textbf{Faithfulness.}}\label{para:faithfulness}

A faithful attribution path should accurately reflect how the model prediction changes when important or unimportant features are perturbed. 
We quantify this by measuring the difference in model prediction confidence under two complementary perturbations, \textit{insertion} and \textit{deletion}. The \textit{faithfulness score} of a path $\gamma$ is defined as
\begin{align}
    \text{Faithfulness}(\gamma)
    = \mathbb{E}_{r \in \mathcal{R}}
    \Big[ \text{Conf}^{\text{ins}}(r) - \text{Conf}^{\text{del}}(r) \Big],
    \label{eq:faithfulness}
\end{align}
where $\mathcal{R}$ is the set of perturbation ratios (e.g., $r \in \{0.1, 0.2, ..., 0.9\}$). 
Here, $\text{Conf}^{\text{ins}}(r)$ and $\text{Conf}^{\text{del}}(r)$ denote the softmax confidence of the model for the correct class when the most important features are inserted and deleted, respectively. 
A higher faithfulness score indicates that the attribution path better captures features that truly drive the decision of the model, leading to a larger drop in confidence when important features are removed and stable confidence when unimportant ones are altered.

\paragraph{\textbf{Complexity.}}

While high faithfulness is desirable, overly complex attributions that rely on many dispersed features may hinder interpretability. 
To regularize this, we introduce a \textit{complexity} measure inspired by attribution entropy~\cite{bhatt2020evaluating}. 
Given an attribution map $\mathcal{A} \in \mathbb{R}^{n}$ for a path $\gamma$, we define normalized weights $p_i = \frac{|\mathcal{A}_i|}{\sum_j |\mathcal{A}_j|}$ and the entropy-based complexity score:
\begin{align}
    \text{Comp}(\gamma)
    = - \sum_{i=1}^{n} p_i \log (p_i + \varepsilon),
    \label{eq:complexity}
\end{align}
where $\varepsilon$ is a small constant added for numerical stability. 
Lower complexity corresponds to sparser, more human-interpretable explanations, while higher complexity indicates diffuse and harder-to-interpret attributions.

\paragraph{\textbf{Combined Guidance.}}

To guide the diffusion model toward generating paths that are both faithful and interpretable, 
we employ two complementary guidance networks: 
$\mathcal{J}_{\phi_1}$ for \textit{faithfulness} and $\mathcal{J}_{\phi_2}$ for \textit{complexity}. 
Each network is independently trained to regress its respective objective from noisy path samples 
following the regression formulation in Eq.~\eqref{Eq:MSE_qt_objective}. 

During the reverse diffusion process, the dual guidance signals are combined into a unified gradient field that jointly influences the path generation:
\begin{align}
\label{eq:combined_guidance}
    g =
    \Big(
    \lambda_{\text{faith}} \nabla_{{\gamma}}
    \mathcal{J}_{\phi_1}({\gamma})
    +
    \lambda_{\text{comp}} \nabla_{{\gamma}}
    \mathcal{J}_{\phi_2}({\gamma})
    \Big)\Big|_{{\gamma}=\boldsymbol{\mu}_{\theta}({\gamma}^\tau)},
\end{align}
where $\lambda_{\text{faith}}$ and $\lambda_{\text{comp}}$ are coefficients balancing the two guidance strengths.
This combined gradient $g$ is injected into the reverse transition of the diffusion model in Eq.~\eqref{eq:guided_sampling}, steering the denoising trajectory toward attribution paths that are both faithful to the model reasoning
and sufficiently simple to interpret.
\edit{In all our main experiments, we fix the global guidance scale $\omega{=}1$ in Eq.~\eqref{eq:guided_sampling} and control the overall guidance strength directly through $\lambda_{\text{faith}}$ and $\lambda_{\text{comp}}$; the effect of varying $\omega$ is examined separately in Appendix~\ref{appx:latent_pixel}.}

\subsection{Practical Algorithm}

\paragraph{\textbf{Path Generation over Latent Space.}}

Generating diffusion paths directly in the pixel space is computationally prohibitive due to the high dimensionality of images. Instead, we employ a pre-trained $\beta$-VAE \cite{higgins2017beta, li2024autoregressive} encoder $\mathcal{E}_\text{enc}$ and decoder $\mathcal{D}_\text{dec}$ to construct paths in a lower-dimensional latent space. Given an input $x$ and its baseline $x'$, we first obtain their latent representations $z = \mathcal{E}_\text{enc}(x)$ and $z' = \mathcal{E}_\text{enc}(x')$. The diffusion model is trained in the latent space to learn the distribution of latent trajectories $\gamma_z(t)$ that, when decoded, yield non-linear paths $\gamma(t)$ in the image space. At inference time, it generates latent trajectories $\gamma_z(t)$, which are decoded back to the image domain via $\mathcal{D}_\text{dec}$:
\begin{align}
    \gamma(t) = \mathcal{D}_\text{dec}(\gamma_z(t)), \quad 
    \text{where } \gamma_z(0) = z', \ \gamma_z(1) = z.
\end{align}

\paragraph{\textbf{Stochastic Multi-path Sampling.}}

Diffusion sampling naturally enables stochastic generation of multiple attribution paths from the guidance-conditioned distribution $\tilde{p}_\theta({\gamma}) \propto p_\theta({\gamma}) \exp(\mathcal{J}({\gamma}))$. To leverage this property, we generate $N$ distinct path candidates 
$\{ \gamma^{(k)}(t) \}_{k=1}^N$ between the baseline $x'$ and the input $x$, 
where the superscript $(k)$ indexes each sampled path. 
We consider two strategies to utilize these multiple paths:

\textit{(i) Best-of-$N$ Search:}
Among the sampled candidates, we select the path that best balances the two guidance objectives of faithfulness and complexity:
\begin{align}
    \gamma^{*} 
    &= \argmax_{\gamma^{(k)}} 
    \Big[
        \lambda_{\text{faith}} \, \mathcal{J}_{\phi_1}(\gamma^{(k)}) 
        + \lambda_{\text{comp}} \, \mathcal{J}_{\phi_2}(\gamma^{(k)}) 
    \Big].
    \label{eq:best_of_n}
\end{align}
The selected path $\gamma^{*}$ is then used to compute the final attribution via Integrated Gradients (IG).

\textit{(ii) Multi-path Aggregation:}
Alternatively, we compute attributions for all sampled paths and aggregate them to obtain a more consistent explanation. Let $\mathcal{A}(x; \gamma^{(k)})$ denote the attribution map computed along a specific path $\gamma^{(k)}$. We then aggregate the path-dependent attributions to form the final attribution map for $x$:
\begin{align}
    \bar{\mathcal{A}}(x)
    = \mathcal{M}\big(\{ \mathcal{A}(x; \gamma^{(k)}) \}_{k=1}^{N} \big),
\end{align}
where $\mathcal{M}$ denotes an aggregation operator such as mean, median, or variance-weighted mean. This aggregation reduces the stochastic variance among sampled paths and yields more reliable and consistent attributions. Further details on each aggregation strategy are provided in Appendix~\ref{appx:aggregation_methods}.

\paragraph{\textbf{Axiomatic Properties.}}
Since DiffIG computes attributions via the path integral in Eq.~\eqref{eq:path_general}, it is a path-based attribution method and naturally inherits the axiomatic properties established for this class of methods~\citep{sundararajan2017axiomatic}. In particular, we formally prove in Appendix~\ref{appx:axioms} that DiffIG satisfies \emph{Completeness} (attributions sum to $f(x) - f(x')$), \emph{Sensitivity} (features with zero partial derivatives receive zero attribution), and \emph{Implementation Invariance} (attributions depend solely on the model's input-output behavior). These guarantees hold for any individual path generated by the diffusion model, and completeness is additionally preserved under mean aggregation of multiple paths.

\section{Experiments}

% In this section, we empirically validate the effectiveness of our proposed DiffIG method. 
% Specifically, we investigate \textbf{(1)} whether DiffIG can generate quantitatively and qualitatively more perceptually aligned explanations, 
In this section, we empirically evaluate the effectiveness of our proposed DiffIG method. Specifically, we examine \textbf{(1)} whether DiffIG produces explanations that are more perceptually aligned in both quantitative and qualitative terms, 
\textbf{(2)} whether its performance can be further enhanced through multi-path sampling, and \textbf{(3)} whether inference-time guidance can effectively steer path generation toward desired objectives. A complete description of our experimental setup and implementation details is provided in Appendices~\ref{appx:experiment_setup} and~\ref{appx:optimal_hyperparameters}.

\subsection{Experimental Setup}\label{subsec:experiment_setup}

\paragraph{\textbf{Datasets and Models}}

We evaluate the effectiveness of DiffIG on two real-world image datasets: 
(1) the Oxford-IIIT Pet dataset~\cite{parkhi2012cats}, which contains images of 37 pet categories, 
and (2) the Mini-ImageNet dataset~\cite{sun2019meta}, a subset of the larger ImageNet corpus comprising 100 categories. For classification, we use publicly available pre-trained models including VGG16~\citep{simonyan2015very}, \edit{InceptionV1}~\citep{szegedy2015going}, and ResNet18~\citep{he2016deep}, which are further finetuned on each dataset.

\paragraph{\textbf{Baselines}}

To demonstrate the effectiveness of DiffIG, we compare it against seven representative path-based attribution methods. Our comparison includes IG~\cite{sundararajan2017axiomatic}; optimization-based methods (GIG~\cite{kapishnikov2021guided}, IG$^2$~\cite{zhuo2024ig2}, and AGI~\cite{pan2021explaining}); stochastic process-based methods (SPI~\cite{jeon2023beyond}); and manifold-informed methods (EIG~\cite{jha2020enhanced}, MIG~\cite{zaher2024manifold}). 

\paragraph{\textbf{Evaluation Metrics}}

To quantitatively evaluate the quality of attribution maps, we adopt the widely used pixel perturbation framework~\citep{petsiuk2018rise}, which evaluates how the model output varies as important pixels are progressively inserted or deleted. Specifically, we use the Insertion and Deletion metrics~\citep{petsiuk2018rise}, and their difference, the DiffID score~\citep{yang2022re}. The Insertion metric computes the area under the confidence curve (AUC) as pixels are added in order of importance, while Deletion does so as pixels are removed. The DiffID score captures the gap between the two AUCs, offering a unified measure of both sensitivity and selectivity, where higher values indicate more faithful and discriminative attributions.

\begin{figure*}[t!]
    \centering
    \includegraphics[width=\linewidth]{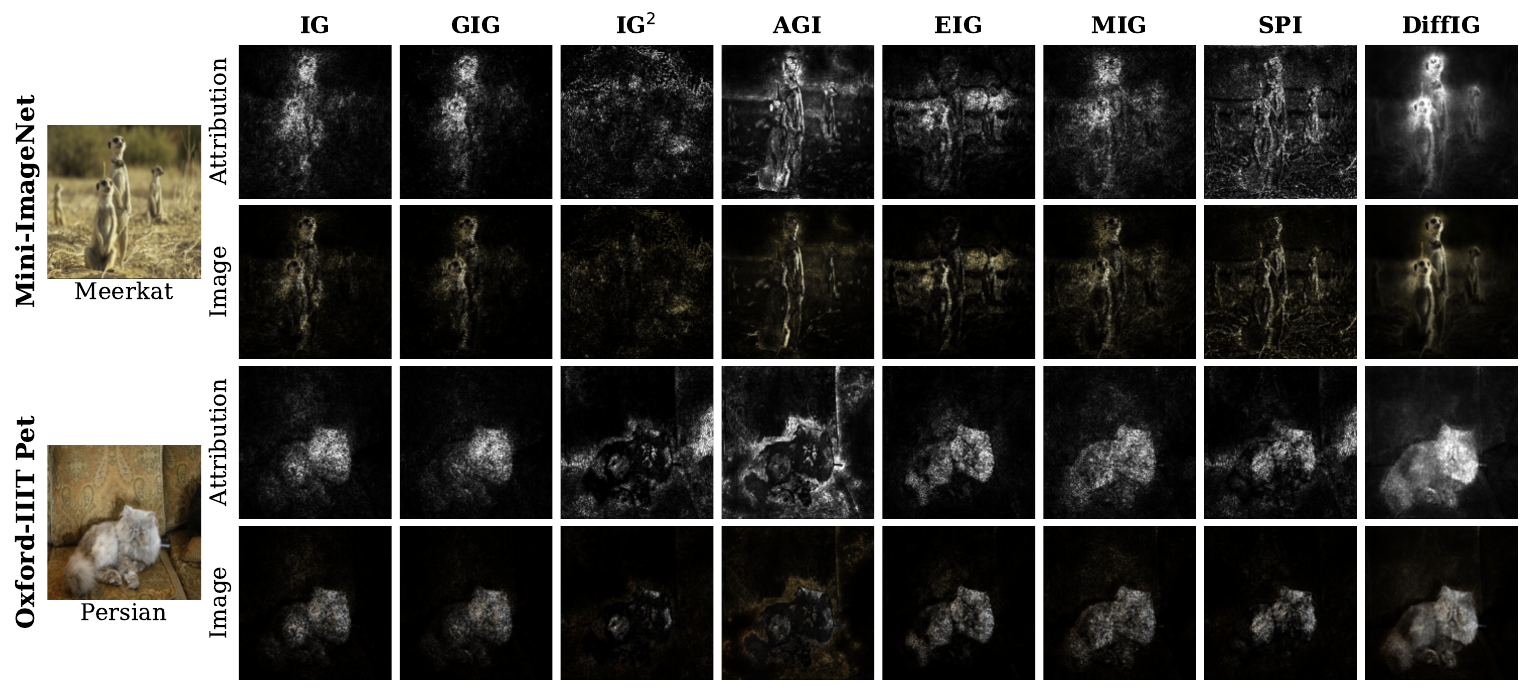}
    \vspace{-0.5cm}
    \caption{\textbf{Qualitative comparison of path-based attribution methods on Oxford-IIIT Pet and Mini-ImageNet using VGG16.} The \textit{Attribution} rows show the attribution maps, and the \textit{Image} rows show them overlaid on the input images.}
    \label{fig:qual}
    \vspace{-0.3cm}
\end{figure*}

\subsection{Qualitative Results}

We qualitatively compare the DiffIG attribution maps with those from other path-based methods. Attribution results are computed for random \edit{validation} images from the Oxford-IIIT Pet and Mini-ImageNet datasets. As illustrated in Figure~\ref{fig:qual}, DiffIG produces attribution maps that more precisely highlight important pixels and more comprehensively localize salient objects. For instance, in the \textit{Meerkat} example, while other methods emphasize only the two central figures, DiffIG additionally identifies the third meerkat on the far right\edit{(Appendix~\ref{appx:meerkat})}. 
% \edit{Since attribution is model-specific, we verify that this far-right region carries genuine classifier evidence rather than DiffIG over-amplification: inserting the region raises the target-class logit and probability, while deleting it lowers the target logit across black, mean, and blur in-painting (Appendix~\ref{appx:meerkat}).}
Furthermore, local optimization approaches such as IG$^2$ and AGI tend to follow spurious path features, leading to misleading attribution aggregation. This issue is evident in the \textit{Persian} example, where these methods are distracted by the high-frequency texture of the sofa, incorrectly assign importance to the background, and create spurious \textit{ghost cat} attributions on the upper right side of the image. 
In comparison, DiffIG accurately attributes importance to both the cat’s face and body, reflecting a more complete and faithful interpretation of the model reasoning process. Additional qualitative results appear in Appendix~\ref{appx:additional_qual}.

\setlength{\intextsep}{0pt}
\begin{wrapfigure}[11]{l}{0.58\linewidth}
    \centering
    \includegraphics[width=\linewidth]{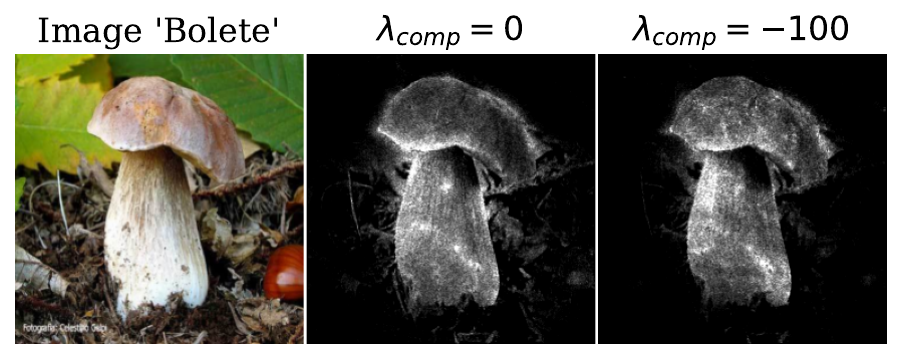}
    \caption{\textbf{Effect of Complexity Guidance.} The attributions for $\lambda_{\text{comp}}=0$ (no guidance) and $-100$ (strong), showing sparser and more focused attributions as the guidance strength increases.}
    % The image sample is from the Mini-ImageNet validation set.
    \label{fig:comp_ablation}
\end{wrapfigure}

\vspace{0.2cm}
\noindent \textbf{Flexible Control of Sparseness.} Figure~\ref{fig:comp_ablation} illustrates that DiffIG enables flexible control over the sparsity of attributions through the complexity guidance. By increasing the guidance strength, DiffIG produces sparser and more focused attributions while effectively suppressing irrelevant background features. 

\subsection{Quantitative Results}
Table~\ref{tab:quant} presents the quantitative comparison of DiffIG with other path-based methods. Our approach consistently matches or surpasses the performance of existing baselines across various datasets and model architectures. We attribute these improvements to the ability of DiffIG to perform \edit{a guided, non-greedy search over generated path candidates}, rather than relying on greedy or local search heuristics as in methods such as GIG, IG$^2$, and AGI. Figure~\ref{fig:ablation_aggregation} further demonstrates the effectiveness of the multi-path sampling in DiffIG. This highlights that generating and utilizing multiple diffusion-sampled paths leads to more stable and consistent attributions.

\begin{table*}[t!]
\centering

\caption{\textbf{Quantitative results across different datasets and model architectures.} 
We compare DiffIG with various path-based attribution methods on the full Oxford-IIIT Pet \edit{validation} set (370 images) and 500 randomly sampled images from the Mini-ImageNet \edit{validation} set. \textbf{Bolded} values indicate the best performance, and \underline{underlined} values indicate the second best.}
\vspace{-0.3cm}
\resizebox{\textwidth}{!}{ 
\begin{tabular}{llrrrrrrrrr}
\toprule
& & \multicolumn{3}{c}{\textbf{ResNet18}} & \multicolumn{3}{c}{\textbf{VGG16}} & \multicolumn{3}{c}{\textbf{\edit{Inception V1}}} \\
\cmidrule(lr){3-5} \cmidrule(lr){6-8} \cmidrule(lr){9-11}
\textbf{Data} & \textbf{Method} & \textbf{DiffID ($\uparrow$)} & \textbf{Ins ($\uparrow$)} & \textbf{Del ($\downarrow$)} & \textbf{DiffID ($\uparrow$)} & \textbf{Ins ($\uparrow$)} & \textbf{Del ($\downarrow$)} & \textbf{DiffID ($\uparrow$)} & \textbf{Ins ($\uparrow$)} & \textbf{Del ($\downarrow$)} \\
\midrule
& IG~\citep{sundararajan2017axiomatic} & 0.3213 & 0.4892 & 0.1679 & 0.4654 & 0.5523 & 0.0869 & 0.3051 & 0.4333 & \underline{0.1282} \\
    \rowcolor{palegray}
\cellcolor{white} &  GIG~\citep{kapishnikov2021guided} & \underline{0.3486} & \underline{0.5065} & \underline{0.1579} & 0.4859 & 0.5659 & \underline{0.0799} & 0.3085 & 0.4376 & 0.1290 \\
& IG$^2$~\cite{zhuo2024ig2} & 0.2767 & 0.4385 & 0.1619 & 0.3016 & 0.4259 & 0.1243 & 0.2228 & 0.3650 & 0.1422 \\
    \rowcolor{palegray}
\cellcolor{white} & AGI~\cite{pan2021explaining} & 0.3242 & 0.4879 & 0.1637 & \underline{0.4930} & \underline{0.5861} & 0.0930 & \underline{0.4092} & \underline{0.5156} & \textbf{0.1064} \\
& EIG~\citep{jha2020enhanced} & 0.2680 & 0.4532 & 0.1852 & 0.3950 & 0.4898 & 0.0947 & 0.2698 & 0.4108 & 0.1410 \\
    \rowcolor{palegray}
\cellcolor{white} & MIG~\citep{zaher2024manifold} & 0.2552 & 0.4316 & 0.1764 & 0.3681 & 0.4723 & 0.1041 & 0.2602 & 0.3963 & 0.1361 \\
& SPI~\citep{jeon2023beyond} &0.2738 & 0.4414 & 0.1676 & 0.4202 & 0.5046 & 0.0844 & 0.2292 & 0.3751 & 0.1459 \\
\cmidrule(lr){2-11}
    \rowcolor{lightgray}
    \multirow{-9}{*}{\rotatebox[origin=c]{90}{\textbf{Oxford-IIIT Pet}}}
\cellcolor{white} & \textbf{DiffIG} (Ours) & \textbf{0.5067} & \textbf{0.6300} & \textbf{0.1233} & \textbf{0.6368} & \textbf{0.7128} & \textbf{0.0760} & \textbf{0.4817} & \textbf{0.6112} & 0.1296 \\
\midrule
& IG~\citep{sundararajan2017axiomatic} & 0.2147 & 0.4522 & 0.2375 & 0.3158 & 0.4786 & 0.1628 & 0.2975 & 0.5168 & 0.2194 \\
    \rowcolor{palegray}
\cellcolor{white} & GIG~\citep{kapishnikov2021guided} & 0.1896 & 0.4369 & 0.2474 & 0.3222 & 0.4803 & 0.1581 & 0.2861 & 0.5152 & 0.2291 \\
& IG$^2$~\cite{zhuo2024ig2} & 0.1505 & 0.3858 & 0.2353 & 0.2251 & 0.4091 & 0.1840 & 0.2446 & 0.4423 & \textbf{0.1977} \\
    \rowcolor{palegray}
\cellcolor{white} & AGI~\cite{pan2021explaining} & \underline{0.2949} & \underline{0.5018} & \textbf{0.2068} & \underline{0.4228} & \underline{0.5695} & \textbf{0.1466} & \textbf{0.3579} & \underline{0.5714} & 0.2135 \\
& EIG~\citep{jha2020enhanced} & 0.1480 & 0.4166 & 0.2686 & 0.2385 & 0.4374 & 0.1988 & 0.2312 & 0.4795 & 0.2483 \\
    \rowcolor{palegray}
\cellcolor{white} & MIG~\citep{zaher2024manifold} & 0.2227 & 0.4398 & 0.2171 & 0.3106 & 0.4660 & \underline{0.1554} & 0.3057 & 0.5162 & \underline{0.2105} \\
\cellcolor{white} & SPI~\citep{jeon2023beyond} & 0.2325 & 0.4491 & \underline{0.2165} & 0.2763 & 0.4463 & 0.1700 & 0.2831 & 0.5112 & 0.2281 \\
\cmidrule(lr){2-11}
    \rowcolor{lightgray}
    \multirow{-9}{*}{\rotatebox[origin=c]{90}{\textbf{Mini-ImageNet}}}
\cellcolor{white} & \textbf{DiffIG} (Ours) & \textbf{0.3126} & \textbf{0.5747} & 0.2621 & \textbf{0.4298} & \textbf{0.6037} & 0.1739 & \underline{0.3578} & \textbf{0.6285} & 0.2707 \\
\bottomrule
\end{tabular}
} 
\label{tab:quant}
\vspace{-0.3cm}
\end{table*}

\setlength{\intextsep}{0pt}
\begin{wrapfigure}{l}{0.5\linewidth}
    % \vspace{-0.5em}
    \vspace{-0.1cm}
    \centering
    \includegraphics[width=\linewidth]{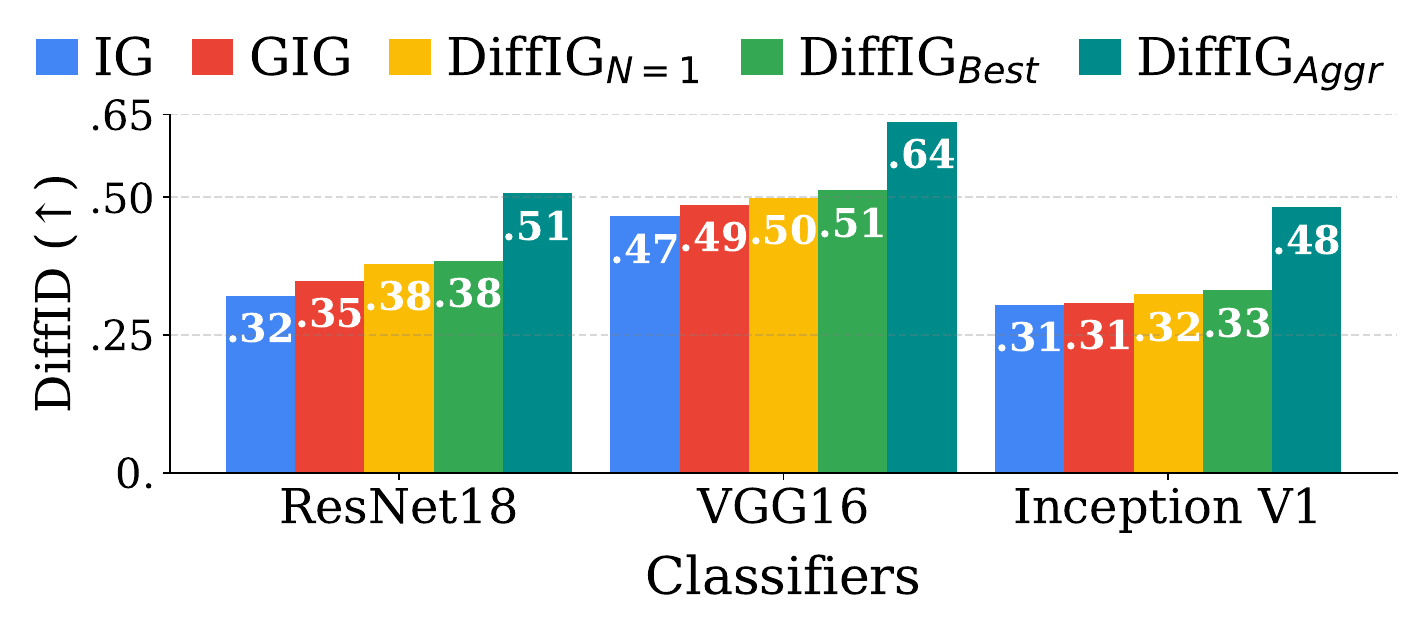}
    % \vspace{-0.3cm}
    \caption{\textbf{Effect of Multi-path Sampling of DiffIG.} The DiffIG variants correspond to \textit{one-path generation (N=1)}, \textit{\edit{Best-of-$N$} Search (Best)}, and \textit{multi-path mean aggregation (Aggr)}.}
    % \vspace{-0.8em}
    \vspace{0.1cm}
    \label{fig:ablation_aggregation}
\end{wrapfigure}

 \paragraph{\textbf{Validity of Guided Generation.}}
DiffIG does not directly optimize the final evaluation pipeline, since the guidance predictor is a proxy regressor trained on noisy latent paths, while the final attribution faithfulness is evaluated later in pixel space. 
We also jointly optimize faithfulness and complexity, which discourages degenerate solutions that exploit a single score.
This is further supported by the Pareto frontier analysis in Appendix~\ref{appx:optimal_hyperparameters}, where the trade-off between Insertion and Deletion confirms that faithfulness improves through balancing multiple objectives. 
\edit{As further evidence that DiffIG does not merely fit the evaluation metric, it ranks first on the independent Faithfulness Correlation metric (computed on absolute attributions) from Quantus~\cite{hedstrom2023quantus}, where higher is better (DiffIG $-0.0025$ vs.\ AGI $-0.0042$ and GIG $-0.0093$ on Oxford-IIIT Pet with ResNet18; see Appendix~\ref{appx:quantus}).}

\paragraph{\textbf{Additional Experiments.}}
We provide further analyses in the appendix.
Appendix~\ref{appx:ddpath_rebuttal} compares DiffIG against DDPath~\cite{lei2024denoising}, a concurrent diffusion-based attribution method, demonstrating consistent improvements in DiffID and Insertion.
Appendix~\ref{appx:vit_rebuttal} evaluates DiffIG on ViT-B/16~\cite{dosovitskiy2021an}, confirming that performance gains generalize beyond CNNs to transformer architectures.
Appendix~\ref{appx:runtime_rebuttal} presents a runtime analysis showing that DiffIG achieves a favorable trade-off between computational cost and attribution quality.
Appendix~\ref{appx:vae_rebuttal} reports sensitivity analysis across different VAE backbones (MAR~\cite{li2024autoregressive}, Stable Diffusion 2.1~\cite{rombach2022high}, Kandinsky 2.1~\cite{razzhigaev2023kandinsky}), confirming robustness to the choice of generative prior.
\edit{Appendix~\ref{appx:path_alpha} visualizes path dependence and shows robustness to the SBP concentration $\alpha$, and Appendix~\ref{appx:baseline_choice} shows robustness to the baseline input, where a black--white mean mitigates the dark-region limitation.
Appendix~\ref{appx:latent_pixel} examines the latent--pixel domain gap on a controlled toy example; Appendix~\ref{appx:smoothing} compares against IG smoothing variants~\cite{smilkov2017smoothgrad}; and Appendices~\ref{appx:denoising_steps} and~\ref{appx:failure_cases} give step-wise visualizations and a failure-case analysis.}

\section{Related Work}
\label{sec:related_work}

\paragraph{\textbf{Path-based Attribution Methods.}}

As discussed in Section~\ref{sec:background_path_dependency}, the path-dependent nature of attributions in highly discontinuous gradient fields poses a key challenge for reliable explanations. 
Therefore, the central problem lies in identifying a desirable or faithful integration path.

One line of research formulates path identification as an optimization problem. For example, Guided IG (GIG)~\citep{kapishnikov2021guided} argues that high-curvature areas of the output surface cause noisy gradients and should be avoided. Geometrically Guided IG (GGIG)~\citep{rahman2022geometrically} selects a path toward class-specific activation maximization. IG$^2$~\citep{zhuo2024ig2} seeks a path with steeper gradients to avoid saturation areas. Adversarial Gradient Integration (AGI)~\citep{pan2021explaining} uses loss minimization paths from adversarial examples. \edit{Spectral IG~\citep{kim2026spectral} instead refines attributions in a coarse-to-fine manner using the singular value decomposition.} However, these methods rely on heuristic, locally optimized search procedures rather than directly optimizing global objectives, which can result in sub-optimal paths that may not be holistically faithful.

Another line of work models the path as a distribution and applies stochastic processes. SAlient Manipulation Path (SAMP)~\citep{zhang2024path} defines the path selection process as an additive stochastic process, akin to Brownian motion, and regards attribution as a cumulative process where contributions are allocated sequentially. Similarly, Stick-breaking Path Integration (SPI)~\citep{jeon2023beyond} also defines a distribution over possible paths using a stick-breaking process. Despite capturing stochastic variability, these methods remain model-agnostic and fail to account for the underlying decision surface.

A recent line of work incorporates data manifold information to produce perceptually faithful attributions. For example, Enhanced Integrated Gradients (EIG)~\cite{jha2020enhanced} and Manifold IG (MIG)~\cite{zaher2024manifold} utilize the learned latent space of a variational autoencoder (VAE) to construct linear or geodesic paths that better respect the underlying data manifold. \edit{Manifold-Aligned Guided IG~\citep{kim2026manifoldaligned} similarly steers the path to remain aligned with the data manifold.} Denoising Diffusion Path (DDPath)~\cite{lei2024denoising} leverages a pre-trained diffusion model to construct the attribution path. The core idea is that the reverse diffusion process, which progressively transforms noise into an image, can be interpreted as a path that stays closer to the data manifold than a simple straight line. While these methods incorporate the data manifold into path generation, they remain constrained by fixed generative processes, lacking flexibility and controllable adaptation in path generation.

\paragraph{\textbf{Planning with Diffusion Models.}}

Diffusion probabilistic models capture complex data distributions by progressively denoising noise into structured samples~\cite{sohl2015deep,ho2020denoising}. This iterative process can be interpreted as learning the \textit{score} of the data distribution~\cite{song2019generative}, with connections to score matching~\cite{hyvarinen2005estimation} and energy-based models (EBMs)~\cite{du2019implicit,nijkamp2019learning,grathwohl2020learning}. In offline reinforcement learning (RL), they have been applied to high-dimensional synthesis~\cite{lu2023synthetic,wang2024prioritized,lee2025state} and to planning~\citep{janner2022planning,lee2023refining,lee2025local,lee2026refining}, where the goal is to synthesize action trajectories achieving high returns. For instance, Diffuser~\citep{janner2022planning} steers an unconditional diffusion model toward high-return trajectories via a return predictor at inference~\citep{dhariwal2021diffusion}, while Decision Diffuser~\citep{ajay2023is} instead applies \emph{classifier-free guidance} to condition directly on reward or constraint signals. Subsequent works extend this paradigm through adaptive fine-tuning~\citep{liang2023adaptdiffuser}, hierarchical control~\citep{li2023hierarchical,chen2024simple}, and multi-agent settings~\citep{zhu2023madiff}.

Inspired by these advances, DiffIG reformulates finding an optimal attribution path as a conditional generative modeling problem, highlighting diffusion models as a unified framework for generating adaptive paths for faithful and interpretable attribution.

\section{Conclusion}

We introduced \textit{Diffusion Integrated Gradients} (DiffIG), a path-based attribution method that formulates integration path generation as a conditional generative modeling problem. By leveraging diffusion models trained on stochastic path distributions with guided sampling at inference, DiffIG enables adaptive, \edit{guided non-greedy search over generated path candidates}, and inference-time controllable path generation, transforming attribution into a flexible generative process steerable toward user-defined objectives. Extensive experiments show that DiffIG matches or outperforms existing path-based methods, producing explanations that are both perceptually aligned and faithful, suggesting that generative modeling is a principled framework for controllable and interpretable attribution.

\paragraph{\textbf{Limitations.}}

DiffIG currently relies on a fixed baseline (e.g., a black image), which can yield less informative attributions for visually dark regions such as dark eyes. \edit{As a simple remedy, averaging the black- and white-baseline attributions mitigates this dark-region issue without retraining(Appendix~\ref{appx:baseline_choice});} adaptive or learned baselines and modalities such as text or tabular data are promising directions. 
% Moreover, developing more effective guidance and alignment mechanisms for diffusion models could further enhance the performance and robustness of DiffIG. 

\section*{\textbf{Acknowledgements}}
This work was supported by the Institute for Information \& Communications Technology Planning \& Evaluation (IITP), funded by the Ministry of Science and ICT (MSIT), Korea, under grants RS-2019-II190075 (Artificial Intelligence Graduate School Program (KAIST)), RS-2022-II220984 (Development of Artificial Intelligence Technology for Personalized Plug-and-Play Explanation and Verification of Explanation), RS-2024-00457882 (AI Research Hub Project), and RS-2026-26644375 (Leading Generative AI Human Resources Development), and by the MSIT InnoCORE program (N10250156).

% \subsection*{Code Availability}
% \label{app:code}
% We provide the PyTorch implementation of our proposed method, Diffusion Integrated Gradients (DiffIG). 
% The implementation consists of the core path generation based on an unconditional diffusion model with classifier guidance and the attribution computation via path integration. 
% The source code of this paper has been made publicly available at \diffigcode.

\bibliographystyle{splncs04}
\bibliography{main}
% --- Supplementary ---
\clearpage

% --- Supplementary Title Start ---
\begin{center}
{\Large \textbf{Supplementary Material for:}}

\medskip

{\Large \textbf{Diffusion Integrated Gradients: Controllable Path Generation for Flexible Feature Attribution}}

\bigskip\bigskip

\large
Soyeon Kim\textsuperscript{1,3}\orcidlink{0009-0001-5037-0902}
\quad
Kyowoon Lee\textsuperscript{2}\orcidlink{0009-0001-3600-8237}\textsuperscript{\dag}
\quad
Jaesik Choi\textsuperscript{1,3}\orcidlink{0000-0002-4663-3263}\textsuperscript{\dag}

\medskip

\normalsize
\textsuperscript{1}KAIST, Korea
\quad
\textsuperscript{2}Ajou University, Korea
\quad
\textsuperscript{3}INEEJI Corp., Korea

\medskip

\texttt{kyowoon.lee1924@gmail.com, \{soyeon.k,jaesik.choi\}@kaist.ac.kr}

\medskip

\textsuperscript{\dag}Co-corresponding authors.

\end{center}
% --- Supplementary Title End ---

\bigskip\bigskip

% toc start
\renewcommand{\contentsname}{\large Table of Contents}
\addtocontents{toc}{\protect\setcounter{tocdepth}{2}}
% --- Supplementary ToC: list appendix sections only.
% Body entries are filtered out of the printed ToC via a marker;
% PDF bookmarks/outline for the main body are unaffected.
\newif\ifshowtocentry \showtocentryfalse
\newcommand{\startappendixtoc}{\showtocentrytrue}
\begingroup
\let\clearpage\relax
\let\oldnumberline\numberline
\renewcommand{\numberline}[1]{\oldnumberline{\textbf{#1}}}
\let\oldcontentsline\contentsline
\renewcommand{\contentsline}[4]{\ifshowtocentry\oldcontentsline{#1}{#2}{#3}{#4}\fi}
\tableofcontents
\endgroup

\appendix
% Give appendix sections hyperref anchors that cannot collide with the
% body's section.1/2/... (fixes ToC links jumping to the main paper under
% pdflatex/Overleaf, where \appendix resets the section counter).
\renewcommand{\theHsection}{appendix.\arabic{section}}
\renewcommand{\theHsubsection}{appendix.\arabic{section}.\arabic{subsection}}
\addtocontents{toc}{\protect\startappendixtoc}
\clearpage
\section{Notations}
\begin{minipage}{\textwidth} 
    \captionsetup{justification=centering, type=figure,hypcap=false}
    \captionof{table}{\textbf{Table of notation}.}
    \label{tab:notation} 
    \scriptsize 
    \centering
    \renewcommand{\arraystretch}{1.1} 
    \resizebox{0.999\textwidth}{!}{ 
    \begin{tabularx}{\textwidth}{ l >{\raggedright\arraybackslash}X }
    \toprule
    \textbf{Notation} & \textbf{Description} \\
    \midrule
    \multicolumn{2}{l}{\textbf{Paths and Attribution}} \\
    \addlinespace
    $x \in \mathbb{R}^n$ & An input to the model, with $n$ features. \\
    $x'$ & A chosen baseline input (e.g., a black or white image). \\
    $f: \mathbb{R}^n \to \mathbb{R}$ & A differentiable model whose prediction we want to explain. \\
    $\gamma(t)$ & An integration path in the input space from $x'$ to $x$, parameterized by $t \in [0, 1]$. \\
    $\dot{\gamma}(t)$ & The path velocity w.r.t. $t$, defined as $\frac{\diff \gamma(t)}{\diff t}$. \\
    $\mathcal{A}_i(\gamma)$ & Attribution score for the $i$-th feature along path $\gamma$. \\
    $\frac{\partial f(\gamma(t))}{\partial \gamma_i(t)}$ & Gradient of model output w.r.t. the $i$-th feature at $\gamma(t)$. \\
    $\frac{\partial \gamma_i(t)}{\partial t}$ & Direction of the path for the $i$-th feature at $\gamma(t)$. \\
    $\gamma_{\text{IG}}(t)$ & The straight-line path for Integrated Gradients, $\gamma_{\text{IG}}(t) = x' + t(x - x')$. \\
    $\gamma^*$ & The optimal path that maximizes an objective $\mathcal{J}(\gamma)$. \\
    $\Gamma$ & The set of all possible paths. \\
    \addlinespace
    \multicolumn{2}{l}{\textbf{Diffusion Probabilistic Models}} \\
    \addlinespace
    $\gamma^0$ & The noiseless data (a clean path). \\
    $\gamma^\tau$ & A path corrupted by noise at diffusion timestep $\tau$. \\
    $\gamma^M$ & A path completely corrupted by Gaussian noise $p(\gamma^M) = \mathcal{N}(0, I)$. \\
    $q(\gamma^\tau | \gamma^{\tau-1})$ & Forward diffusion process (adds noise). \\
    $p_\theta(\gamma^{\tau-1} | \gamma^\tau)$ & Reverse diffusion process (denoises), parameterized by $\theta$. \\
    $\mu_\theta(\gamma^\tau)$ & Mean of the reverse transition distribution. \\
    $\Sigma^\tau$ & Fixed covariance matrix at timestep $\tau$. \\
    $\boldsymbol{\epsilon}$ & Gaussian noise, $\boldsymbol{\epsilon} \sim \mathcal{N}(0, I)$. \\
    $\boldsymbol{\epsilon}_\theta(\gamma^\tau)$ & Neural network predicting noise $\boldsymbol{\epsilon}$ from $\gamma^\tau$. \\
    $p_\theta(\gamma | y = \mathcal{J}(\gamma))$ & Conditional distribution of paths given their objective score. \\
    $\mathcal{D}$ & Dataset of training paths. \\
    \addlinespace
    \multicolumn{2}{l}{\textbf{Objective Functions and Guidance}} \\
    \addlinespace
    $\mathcal{J}(\gamma)$ & Scalar objective quantifying path quality. \\
    $\mathcal{J}_\phi$ & Regression network (params $\phi$) predicting $\mathcal{J}(\gamma)$ from noisy path $\gamma^\tau$. \\
    $\tilde{p}_\theta(\gamma^0)$ & Perturbed distribution, $\tilde{p}_\theta(\gamma^0) \propto p_\theta(\gamma^0)\exp(\mathcal{J}(\gamma^0))$. \\
    $g$ & Guidance gradient, $g = \nabla_\gamma \mathcal{J}_\phi(\gamma)|_{\gamma=\mu_\theta(\gamma^\tau)}$ (combined faithfulness--complexity form in Eq.~\eqref{eq:combined_guidance}). \\
    $\ell_{\text{noise}}$ & Penalizes gradient magnitude on irrelevant features. \\
    $\ell_{\text{distance}}$ & Measures deviation from linear path. \\
    $\mathcal{E}_{\text{enc}}(\cdot)$, $\mathcal{D}_{\text{dec}}(\cdot)$ & Pretrained $\beta$-VAE encoder/decoder. \\
    $z, z'$ & Latent codes, $z = \mathcal{E}_{\text{enc}}(x)$, $z' = \mathcal{E}_{\text{enc}}(x')$. \\
    $\gamma_z(t)$ & Latent trajectory. \\
    $\mathcal{J}_{\phi_1}(\gamma)$ & Faithfulness guidance network. \\
    $\mathcal{J}_{\phi_2}(\gamma)$ & Complexity guidance network. \\
    $\text{Faithfulness}(\gamma)$ & Faithfulness score of a path. \\
    $\text{Conf}_{\text{ins}}(r), \text{Conf}_{\text{del}}(r)$ & Model confidence after insertion/deletion by ratio $r$. \\
    $\mathcal{R}$ & Set of perturbation ratios. \\
    $\mathcal{A} \in \mathbb{R}^n$ & Attribution map. \\
    $p_i$ & Normalized attribution weights, $p_i = |\mathcal{A}_i| / \sum_j |\mathcal{A}_j|$. \\
    $\text{Comp}(\gamma)$ & Entropy-based complexity score. \\
    $\varepsilon$ & Small constant for numerical stability. \\
    \addlinespace
    \multicolumn{2}{l}{\textbf{Multi-path Sampling and Aggregation}} \\
    \addlinespace
    $N$ & Number of sampled paths. \\
    $\gamma^{(k)}(t)$ & $k$-th sampled path. \\
    $\mathcal{A}(x; \gamma^{(k)})$ & Attribution map computed along $\gamma^{(k)}$. \\
    $\bar{\mathcal{A}}(x)$ & Final aggregated attribution map. \\
    $\mathcal{M}$ & Aggregation operator (e.g., mean or median). \\
    $s$ & SPI-P attribution threshold (95th percentile of mean pixel values). \\
    \bottomrule
    \end{tabularx}
    }
\end{minipage}

\clearpage

\section{Axiomatic Properties of DiffIG}\label{appx:axioms}

DiffIG computes attributions by integrating model gradients along a continuous path $\gamma(t)$ from the baseline $x'$ to the input $x$, as formalized in Eq.~\eqref{eq:path_general}. Since DiffIG adheres to this path integral formulation, it constitutes a path-based attribution method and inherits the core axiomatic properties established for this class of methods~\cite{sundararajan2017axiomatic, friedman2004paths}. We formally verify these properties below.

\begin{proposition}[Completeness]\label{prop:completeness}
	Let $f: \mathbb{R}^n \to \mathbb{R}$ be a differentiable model and let $\gamma: [0,1] \to \mathbb{R}^n$ be a piecewise-smooth path with $\gamma(0) = x'$ and $\gamma(1) = x$. Then the attributions defined by Eq.~\eqref{eq:path_general} satisfy:
	\begin{equation}
		\sum_{i=1}^{n} \mathcal{A}_i(\gamma) = f(x) - f(x').
	\end{equation}
\end{proposition}

\begin{proof}
	Aggregating the per-feature attributions yields:
	\begin{align}
		\sum_{i=1}^{n} \mathcal{A}_i(\gamma)
		 & = \int_{0}^{1} \sum_{i=1}^{n} \frac{\partial f(\gamma(t))}{\partial \gamma_i(t)} \, \frac{\diff \gamma_i(t)}{\diff t} \diff t
		= \int_{0}^{1} \nabla f(\gamma(t))^\top \dot{\gamma}(t) \, \diff t.
	\end{align}
	Recognizing the integrand as the total derivative $\frac{\diff}{\diff t} f(\gamma(t))$ via the chain rule, a direct application of the fundamental theorem of calculus gives
	$\int_{0}^{1} \frac{\diff}{\diff t} f(\gamma(t)) \, \diff t = f(\gamma(1)) - f(\gamma(0)) = f(x) - f(x')$.
\end{proof}

\begin{remark}[\textbf{Completeness under Latent-Space Path Generation}]
	\label{remark:completeness_latent}
	DiffIG constructs paths in the latent space of a pre-trained $\beta$-VAE, where a latent trajectory $\gamma_z(t)$ is decoded to the input space via $\tilde{\gamma}(t) = \mathcal{D}_\text{dec}(\gamma_z(t))$. Because the autoencoder introduces a reconstruction gap, the decoded endpoints may deviate from the true $x'$ and $x$. To guarantee strict completeness, DiffIG explicitly enforces the boundary conditions $\gamma(0) = x'$ and $\gamma(1) = x$ in the input space during path construction, while relying on the decoded latent trajectory for intermediate points $0 < t < 1$. Since the fundamental theorem of calculus depends only on the endpoint values, this strategy ensures that completeness holds exactly, irrespective of the decoder's reconstruction quality.
\end{remark}

\begin{remark}[\textbf{Completeness under Multi-path Aggregation}]
	\label{remark:completeness_multipath}
	When $N > 1$ paths are sampled and combined via the mean operator, completeness is preserved by linearity. Specifically, let $\bar{\mathcal{A}}_i = \frac{1}{N}\sum_{k=1}^{N} \mathcal{A}_i(x; \gamma^{(k)})$. Since each path independently satisfies completeness, $\sum_i \bar{\mathcal{A}}_i = \frac{1}{N}\sum_{k=1}^{N} (f(x) - f(x')) = f(x) - f(x')$. Non-linear operators such as the median do not preserve this equality in general; however, they may be preferred when robustness to outlier paths outweighs the strict completeness guarantee.
\end{remark}

\begin{proposition}[Sensitivity]
	\label{prop:sensitivity}
	If a feature $x_i$ has no influence on the model output, \ie, $\frac{\partial f(x)}{\partial x_i} = 0$ for every $x \in \mathbb{R}^n$, then $\mathcal{A}_i(\gamma) = 0$ for any admissible path $\gamma$.
\end{proposition}

\begin{proof}
	By hypothesis, the gradient component $\frac{\partial f(\gamma(t))}{\partial \gamma_i(t)}$ vanishes for all $t \in [0,1]$. The integrand in Eq.~\eqref{eq:path_general} is therefore identically zero, and the attribution reduces to $\mathcal{A}_i(\gamma) = 0$.
\end{proof}

\begin{proposition}[Implementation Invariance]
	\label{prop:impl_invariance}
	Let $f_1$ and $f_2$ be two functionally equivalent networks, \ie, $f_1(x) = f_2(x)$ for all $x \in \mathbb{R}^n$. Then, for any path $\gamma$, the attributions produced by DiffIG coincide: $\mathcal{A}_i^{f_1}(\gamma) = \mathcal{A}_i^{f_2}(\gamma)$ for every feature $i$.
\end{proposition}

\begin{proof}
	Functional equivalence entails $\nabla f_1(x) = \nabla f_2(x)$ for all $x$ at which both gradients exist. The attribution in Eq.~\eqref{eq:path_general} is determined entirely by the gradient field of $f$ evaluated along $\gamma$ and the path velocity $\dot{\gamma}(t)$. Because both quantities are shared between $f_1$ and $f_2$ for any common path $\gamma$, the resulting integrals are identical.
\end{proof}

\section{Experimental Setup}\label{appx:experiment_setup}

In this section, we describe the experimental setup used to evaluate our proposed method. We outline the datasets, models, baseline methods, and evaluation metrics.

\subsection{Datasets and Models}

We conduct our experiments on two standard image classification datasets: Oxford-IIIT Pet~\cite{parkhi2012cats} and a modified version of Mini-ImageNet~\cite{sun2019meta}.

\begin{itemize} 
    \item \textbf{Oxford-IIIT Pet.} This dataset consists of 7,390 images across 37 pet breed categories. We create a validation set by holding out 5\% of the images (370 images) and use the remaining 7,020 images for training.

    \item \textbf{Mini-ImageNet.} This dataset is a subset of ImageNet2012~\cite{deng2009imagenet}, originally curated by~\citet{sun2019meta} for few-shot learning. The standard few-shot split contains disjoint classes for training, validation, and testing. As our focus is on standard classification explanation, we modify this setup. We utilize the WordNet IDs of all 100 available classes and create a comprehensive split where all 100 classes are present in both training and validation sets. This results in 127,090 training images and 6,689 validation images. \edit{For the quantitative evaluation in Table~\ref{tab:quant}, we sample 500 images from these 6,689 validation images using seed 0, while the Oxford-IIIT Pet validation set (370 images) is evaluated in full.}
\end{itemize}

For both datasets, all images are resized to 256$\times$256 and normalized using the standard ImageNet mean [0.485, 0.456, 0.406] and standard deviation [0.229, 0.224, 0.225].

We evaluate explanations using three widely adopted architectures: \textbf{VGG16}~\cite{simonyan2015very}, \textbf{\edit{Inception V1}}~\cite{szegedy2015going}, and \textbf{ResNet18}~\cite{he2016deep}. All models are pretrained on ImageNet and further finetuned on each dataset.

\subsection{Baselines} 

We compare DiffIG against a comprehensive suite of path-based attribution methods. While their conceptual distinctions are discussed in Section~\ref{sec:related_work}, here we describe the specific implementations and hyperparameters used for reproducibility.

For fair comparison, we adopt a zero-value (black image) baseline for all methods and define the explanation objective as the model’s softmax probability for the predicted class. Hyperparameters are chosen based on official implementations or recommendations from the respective papers.

\begin{itemize} 
    \item \textbf{Integrated Gradients (IG)}\footnote{\url{https://github.com/PAIR-code/saliency/blob/master/saliency/core/integrated_gradients.py}}~\cite{sundararajan2017axiomatic}: We use 50 steps for the path integration. 
    \item \textbf{Guided IG (GIG)}\footnote{\url{https://github.com/PAIR-code/saliency/blob/master/saliency/core/guided_ig.py}}~\cite{kapishnikov2021guided}: We use 200 integration steps, a feature fraction of 0.25, and a maximum distance of 0.02. 
    \item \textbf{Manifold IG (MIG)}\footnote{\url{https://github.com/eszaher/Manifold-Integrated-Gradients}}~\cite{zaher2024manifold} and \textbf{Enhanced IG (EIG)}\footnote{We implemented EIG by adapting the conceptually analogous and publicly available code from MIG~\cite{zaher2024manifold}.}~\cite{jha2020enhanced}: Implemented with 20 integration steps. 
    \item \textbf{Adversarial Gradient Integration (AGI)}\footnote{\url{https://github.com/pd90506/AGI}}~\cite{pan2021explaining}: We use 15 steps, 15 negative classes, and a step size of 0.05. 
    \item \textbf{IG$^2$}\footnote{\url{https://github.com/JoeZhuo-ZY/IG2}}~\cite{zhuo2024ig2}: We use 201 steps and a step size of 256\edit{ (in the 0--255 pixel-intensity scale, following the official implementation)}.
    \item \textbf{Stick-breaking Path Integration (SPI)}\footnote{\url{https://openreview.net/attachment?id=muHaELT29WK&name=supplementary_material}}~\cite{jeon2023beyond}: We sample 30 paths with 30 steps each and an alpha of 10.0 for the stick-breaking process. 
    \item \textbf{DiffIG (ours)}: We use 20 integration steps.
\end{itemize} 
All methods were implemented using their official public repositories where available; otherwise, we implemented them faithfully according to their published descriptions.

\subsection{Evaluation Metrics}

To quantitatively evaluate the quality of the generated attribution maps, we use perturbation-based metrics that measure how the model output changes when pixels identified as important are added or removed. We employ three complementary metrics:

\begin{itemize} 
\item \textbf{DiffID}: As introduced in Section~\ref{subsec:experiment_setup} and by~\citet{yang2022re,kasmi2025one}, this metric measures the difference between the Area Under the Curve (AUC) of the Insertion and Deletion curves. Because Insertion increases confidence as important pixels are added and Deletion decreases confidence as the same pixels are removed, the DiffID score summarizes the overall sensitivity and selectivity of an attribution map. Higher values indicate more discriminative and informative attributions.
\item \textbf{Deletion}: Introduced by~\citet{petsiuk2018rise}, this metric measures how quickly the model confidence decreases as the most attribution-relevant pixels are progressively removed (e.g., set to zero). An effective attribution map identifies pixels whose removal leads to a rapid reduction in confidence. We report the AUC of the confidence curve, where a lower AUC indicates a more informative attribution.
\item \textbf{Insertion}: This metric begins from a zero (black) baseline image and progressively inserts the pixels deemed most important. An effective attribution map highlights features whose addition produces a rapid increase in model confidence. We report the AUC of this curve, where a higher AUC corresponds to more informative attributions.
\end{itemize}

\section{Multi-path Aggregation Strategies}
\label{appx:aggregation_methods}

DiffIG can sample $N$ distinct path candidates $\{ \gamma^{(k)} \}_{k=1}^N$ from the guidance-conditioned distribution. For each path $\gamma^{(k)}$, we compute a corresponding attribution map $\mathcal{A}^{(k)} = \mathcal{A}(x; \gamma^{(k)})$. These $N$ maps are then combined into a single, more robust attribution map $\bar{\mathcal{A}}$ using an aggregation operator $\mathcal{M}$. We detail the primary operators evaluated below, where $i$ indexes a specific pixel or feature.

\paragraph{Mean}
This is the straightforward pixel-wise average across all $N$ attribution maps. This operator provides a stable estimate but can be sensitive to outlier paths that contribute noisy attributions.
$$
\bar{\mathcal{A}}_i = \frac{1}{N} \sum_{k=1}^N \mathcal{A}^{(k)}_i
$$

\paragraph{Median}
This operator takes the pixel-wise median value across the $N$ maps. This approach is inherently robust to outliers, as it effectively ignores attribution values from a minority of highly non-optimal or noisy paths.
$$
\bar{\mathcal{A}}_i = \text{median}\left(\left\{ \mathcal{A}^{(k)}_i \right\}_{k=1}^N\right)
$$

\paragraph{Variance-Weighted Mean}
This method assigns a weight to each path's attribution map, where the weight is inversely proportional to the variance of the map itself. This strategy gives more influence to \textit{stable} paths that produce less noisy or more concentrated attributions. Let $w^{(k)} = 1 / (\text{Var}(\mathcal{A}^{(k)}) + \epsilon)$, where $\text{Var}(\cdot)$ is the variance across all pixels in a single attribution map. The final attribution is the weighted average:
$$
\bar{\mathcal{A}} = \frac{\sum_{k=1}^N w^{(k)} \mathcal{A}^{(k)}}{\sum_{k=1}^N w^{(k)}}
$$

\paragraph{SPI-P}
SPI-P operator~\cite{jeon2023beyond} is a probabilistic aggregation strategy. It first computes the pixel-wise mean $\mu_i$ and standard deviation $\sigma_i$ across all $N$ attribution maps. Then, a global attribution threshold $s$ is determined by taking the 95th percentile of all mean pixel values $\{ \mu_i \}$. The final attribution for each pixel $i$, $\bar{\mathcal{A}}_i$, is defined as the \textit{probability} that the attribution at that pixel exceeds the threshold $s$, assuming a Gaussian distribution defined by $\mu_i$ and $\sigma_i$:
$$
\bar{\mathcal{A}}_i = P(\mathcal{A}_i > s) = 1 - \Phi\left(\frac{s - \mu_i}{\sigma_i + \varepsilon}\right)
$$
where $\Phi$ is the standard Gaussian Cumulative Distribution Function (CDF) and $\varepsilon$ is a small constant for numerical stability.

\section{Comparison to Diffusion-Based Attribution}

\label{appx:ddpath_rebuttal}
As no official implementation of DDPath~\cite{lei2024denoising} is publicly available, we evaluated it using our own re-implementation. For fair comparison, both methods use a black baseline and a path length of 20. 
Unlike DDPath, which relies on the vanilla reverse denoising trajectory as its attribution path, DiffIG formulates the path as a guided stochastic process. By applying inference-time guidance, our method steers the trajectories toward regions that maximize faithfulness to the underlying model's decision, rather than strictly following the generative prior. 
Table~\ref{tab:ddpath_rebuttal} shows that DiffIG consistently improves DiffID and Insertion across all three architectures.

\begin{table}[b!]
\caption{\textbf{Comparison with DDPath on Mini-ImageNet.}}
\label{tab:ddpath_rebuttal}
\centering
\begin{tabular}{lccccccccc}
\toprule
\multirow{3}{*}{\textbf{Method}} & \multicolumn{3}{c}{\textbf{ResNet18}} & \multicolumn{3}{c}{\textbf{VGG16}} & \multicolumn{3}{c}{\textbf{\edit{Inception V1}}} \\
\cmidrule(r){2-4} \cmidrule(lr){5-7} \cmidrule(l){8-10}
 & \textbf{DiffID} & \textbf{Ins} & \textbf{Del} & \textbf{DiffID} & \textbf{Ins} & \textbf{Del} & \textbf{DiffID} & \textbf{Ins} & \textbf{Del} \\
\midrule
\textbf{DDPath}~\cite{lei2024denoising} & 0.2582 & 0.4818 & \textbf{0.2236} & 0.3179 & 0.4703 & \textbf{0.1524} & 0.2876 & 0.5096 & \textbf{0.2220} \\
\textbf{DiffIG} & \textbf{0.3126} & \textbf{0.5747} & 0.2621 & \textbf{0.4298} & \textbf{0.6037} & 0.1739 & \textbf{0.3578} & \textbf{0.6285} & 0.2707 \\
\bottomrule
\end{tabular}
% }
\end{table}

\begin{figure}[b!]
    \centering
    \includegraphics[width=\linewidth]{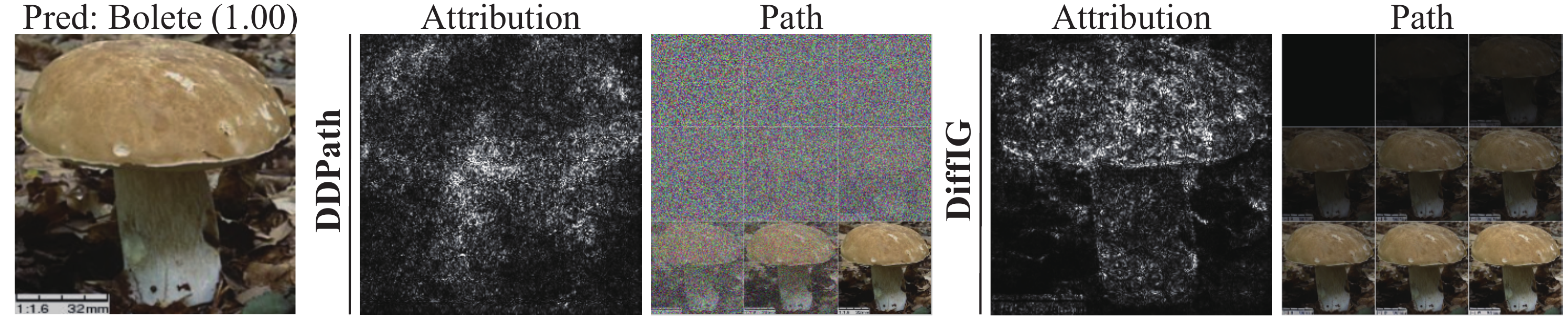}
    \caption{\textbf{Qualitative comparison with DDPath (Mini-ImageNet, ResNet18).} DiffIG produces less noisy attribution maps.}
    \label{fig:ddpath_rebuttal}
\end{figure}

\section{Generalization to Vision Transformer-Based Classifier}
\label{appx:vit_rebuttal}
To test architectural generalization beyond CNNs, we evaluate ViT-B/16~\cite{dosovitskiy2021an} on \edit{Oxford-IIIT Pet}, using the same black baseline and perturbation-based faithfulness protocol. 
As shown in Table~\ref{tab:vit_rebuttal}, DiffIG outperforms all baselines on DiffID, Insertion, and Deletion, indicating that its gains generalize to transformer backbones.

\begin{table}[t!]
    % \scriptsize
    \setlength{\tabcolsep}{10pt}
    \caption{\textbf{Faithfulness scores on ViT-B/16~\cite{dosovitskiy2021an} using the \edit{Oxford-IIIT Pet} dataset.}}
    \label{tab:vit_rebuttal}
    \centering
    \begin{tabular}{l ccc}
        \toprule
        \textbf{Method} & \textbf{DiffID ($\uparrow$)} & \textbf{Insertion ($\uparrow$)} & \textbf{Deletion ($\downarrow$)} \\
        \midrule
        IG~\cite{sundararajan2017axiomatic} & 0.3303 & 0.6775 & 0.3472 \\
        GIG~\cite{kapishnikov2021guided} & 0.3319 & 0.6781 & 0.3462 \\
        IG$^2$~\cite{zhuo2024ig2} & 0.2030 & 0.6010 & 0.3980 \\
        AGI~\cite{pan2021explaining} & 0.4033 & 0.7290 & 0.3257 \\
        EIG~\cite{jha2020enhanced} & 0.3316 & 0.6688 & 0.3373 \\
        MIG~\cite{zaher2024manifold} & 0.3356 & 0.6751 & 0.3395 \\
        SPI~\cite{jeon2023beyond} & 0.1980 & 0.6105 & 0.4125 \\
        \midrule
        \textbf{DiffIG} & \textbf{0.5711} & \textbf{0.7878} & \textbf{0.2166} \\
        \bottomrule
    \end{tabular}
\end{table}

\section{Computational Complexity}
\label{appx:runtime_rebuttal}
We compare against IG, EIG, GIG, AGI, MIG, IG$^2$, and SPI using the same perturbation-metric pipeline. Inference timing is measured on the NVIDIA H100 80GB GPU and averaged over 10 samples. 
Table~\ref{tab:runtime_rebuttal} reports throughput and faithfulness together. 
DiffIG is slower than single-step methods such as IG, but remains competitive with optimization-heavy baselines while achieving stronger faithfulness. 
As the number of sampled paths $N$ increases, faithfulness improves, albeit at an additional runtime cost; we also observe a practical trade-off around 20 steps.
\edit{In terms of memory, the peak GPU usage (allocated/reserved, batch size~1) is $7.0/9.6$~GB for DiffIG$_{N=1}$ and $35.2/59.5$~GB for DiffIG$_{N=10}$, compared with $0.24$~GB for single-pass gradient methods such as IG, GIG, and AGI; the multi-path footprint scales with $N$ and can be reduced by sampling paths sequentially.}
For training, an \edit{Oxford-IIIT Pet}/ResNet18 run required 673.44 PFLOPs for the diffusion model and 171.32 PFLOPs for guidance predictors (0.9 GPU hours on one NVIDIA H100)\edit{, a one-time cost per (dataset, model) pair}.

\begin{table}[t!]
    \setlength{\tabcolsep}{4pt}
    \caption{\textbf{Runtime analysis on \edit{Oxford-IIIT Pet} with ResNet18.} Throughput is averaged over 10 samples (mean $\pm$ std).}
    \label{tab:runtime_rebuttal}
    \centering
    \begin{tabular}{l ccc c}
        \toprule
        \textbf{Method} & \textbf{DiffID ($\uparrow$)} & \textbf{Insertion ($\uparrow$)} & \textbf{Deletion ($\downarrow$)} & \textbf{Seconds / Image} \\
        \midrule
        IG~\cite{sundararajan2017axiomatic} & 0.3213 & 0.4892 & 0.1679 & \textbf{0.0894} $\pm$ 0.0034 \\
        GIG~\cite{kapishnikov2021guided} & 0.3486 & 0.5065 & 0.1579 & 0.1448 $\pm$ 0.0054 \\
        IG$^2$~\cite{zhuo2024ig2} & 0.2767 & 0.4385 & 0.1619 & 2.7705 $\pm$ 0.2893 \\
        AGI~\cite{pan2021explaining} & 0.3242 & 0.4879 & 0.1637 & 0.2941 $\pm$ 0.0312 \\
        EIG~\cite{jha2020enhanced} & 0.2680 & 0.4532 & 0.1852 & 0.0959 $\pm$ 0.0017 \\
        MIG~\cite{zaher2024manifold} & 0.2552 & 0.4316 & 0.1764 & 1.9328 $\pm$ 0.8579 \\
        SPI~\cite{jeon2023beyond} & 0.2738 & 0.4414 & 0.1676 & 21.5796 $\pm$ 0.6048 \\
        \midrule
        \textbf{DiffIG}$_{N=1}$ & 0.3791 & 0.5201 & 0.1410 & 1.1386 $\pm$ 0.0602 \\
        \textbf{DiffIG}$_{N=10}$ & 0.4898 & 0.6143 & 0.1245 & 1.6492 $\pm$ 0.1374 \\
        \textbf{DiffIG}$_{N=30}$ & 0.4940 & 0.6222 & 0.1254 & 2.9701 $\pm$ 0.0431 \\
        \textbf{DiffIG}$_{N=50}$ & \textbf{0.5067} & \textbf{0.6300} & \textbf{0.1233} & 4.5301 $\pm$ 0.0612 \\
        \midrule
        \textbf{DiffIG}$_{\text{step}=10}$ & 0.4889 & 0.6124 & 0.1235 & 3.0531 $\pm$ 0.0621 \\
        \textbf{DiffIG}$_{\text{step}=20}$ & \textbf{0.5067} & \textbf{0.6300} & 0.1233 & 4.5301 $\pm$ 0.0612 \\
        \textbf{DiffIG}$_{\text{step}=40}$ & 0.5011 & 0.6223 & \textbf{0.1212} & 7.6183 $\pm$ 0.2025 \\
        \bottomrule
    \end{tabular}
\end{table}

\section{Robustness to Generative Priors and VAE Backbones}
\label{appx:vae_rebuttal}
We evaluate VAE backbones from MAR, Stable Diffusion 2.1~\cite{rombach2022high} (SD), and Kandinsky 2.1~\cite{razzhigaev2023kandinsky} (KD). 
Table~\ref{tab:vae_rebuttal} shows that DiffIG remains competitive across all choices, indicating robustness to the VAE backbone rather than dependence on a single encoder.

Importantly, models such as SD and KD are pre-trained on general large-scale datasets (e.g., LAION~\citep{schuhmann2022laion}). Applying these models to specific target domains, such as the \edit{Oxford-IIIT Pet} dataset, can introduce out-of-distribution (OOD) challenges. Methods that rely strictly on the learned data manifold of the diffusion model (e.g., DDPath) may be more susceptible to performance degradation in such OOD scenarios. In contrast, our approach appears to mitigate this issue. Because DiffIG actively steers the attribution trajectory using a guidance classifier rather than passively following the unconditional diffusion path, it is less constrained by the training distribution of the pre-trained diffusion model, successfully maintaining high attribution quality even across domain shifts.

\begin{table}[t!]
    \caption{\textbf{Sensitivity analysis on VAE backbone (\edit{Oxford-IIIT Pet}).}}
    \label{tab:vae_rebuttal}
    \centering
    \begin{tabular}{lccccccccc}
    \toprule
    \multirow{3}{*}{\textbf{Method}} & \multicolumn{3}{c}{\textbf{ResNet18}} & \multicolumn{3}{c}{\textbf{VGG16}} & \multicolumn{3}{c}{\textbf{\edit{Inception V1}}} \\
    \cmidrule(r){2-4} \cmidrule(lr){5-7} \cmidrule(l){8-10}
     & \textbf{DiffID} & \textbf{Ins} & \textbf{Del} & \textbf{DiffID} & \textbf{Ins} & \textbf{Del} & \textbf{DiffID} & \textbf{Ins} & \textbf{Del} \\
    \midrule
    \textbf{DiffIG$_{\text{MAR}}$}  & \textbf{0.5067} & \textbf{0.6300} & \textbf{0.1233} & 0.6368 & 0.7128 & \textbf{0.0760} & \textbf{0.4817} & 0.6112 & \textbf{0.1296} \\
    \textbf{DiffIG$_{\text{SD}}$}   & 0.4829 & 0.6094 & 0.1265 & \textbf{0.6772} & \textbf{0.7936} & 0.1164 & 0.4718 & 0.6100 & 0.1381 \\
    \textbf{DiffIG$_{\text{KD}}$}   & 0.4858 & 0.6098 & 0.1240 & 0.6725 & 0.7899 & 0.1174 & 0.4801 & \textbf{0.6176} & 0.1376 \\
    \bottomrule
    \end{tabular}
\end{table}

\section{\edit{Path Dependence and Sensitivity to the SBP Concentration}}
\label{appx:path_alpha}
\edit{A central premise of DiffIG is that the integration path itself, not only the baseline input, materially changes the resulting attribution. Figure~\ref{fig:path_alpha} makes this explicit. On Oxford-IIIT Pet with ResNet18 (100 validation images), we draw Stick-Breaking Process (SBP) paths and evaluate the single-path DiffID for several concentration parameters $\alpha$. The right panel shows that four random SBP paths drawn for the \emph{same} input yield visibly different attribution maps, confirming that the path is a genuine source of variation and motivating a learned, controllable path generator rather than a single hand-crafted path.}

\edit{At the same time, DiffIG is robust to the specific $\alpha$ used to generate its synthetic training paths. The left panel reports the distribution of single-path DiffID across 100 SBP paths for each $\alpha \in \{1, 5, 10, 20\}$; for every $\alpha$ the per-$\alpha$ median stays above the linear-path baseline. Because the synthetic training set is generated by sampling $\alpha$ over a range (Table~\ref{tab:hyperparams_diffig}), this indicates that the final attribution quality is not sensitive to the diversity or precise range of the SBP training distribution.}

\begin{figure}[t!]
    \centering
    \includegraphics[width=\linewidth]{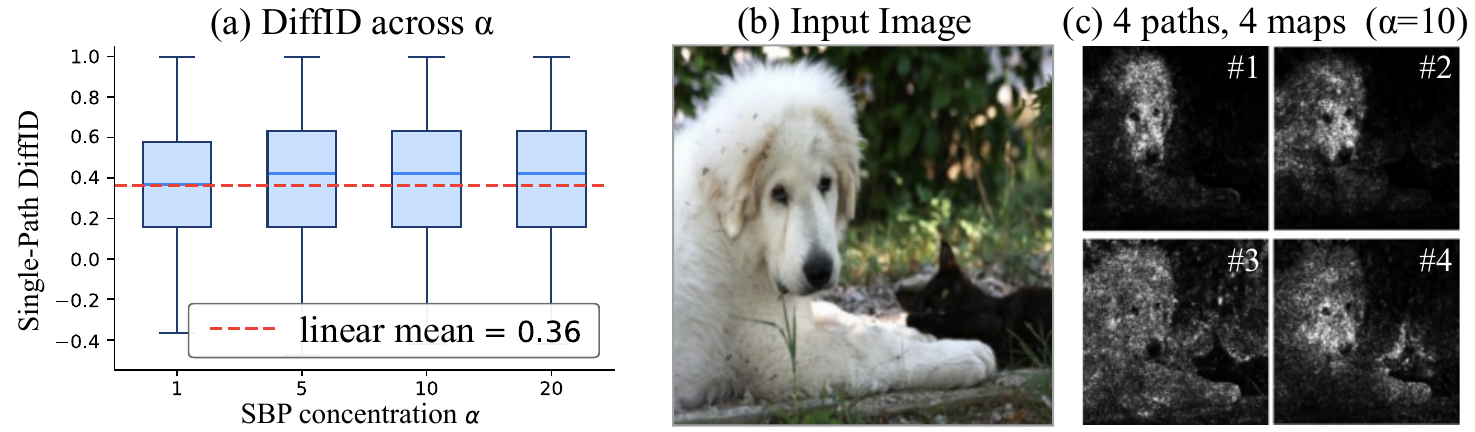}
    \caption{\edit{\textbf{Path dependence and robustness to the SBP concentration $\alpha$} (Oxford-IIIT Pet on ResNet18, 100 validation images). \emph{Left:} single-path DiffID over 100 SBP paths per $\alpha$; the median is above the linear-path baseline for every $\alpha \in \{1, 5, 10, 20\}$. \emph{Right:} four random SBP paths at $\alpha{=}10$ on the same input produce visibly different attributions.}}
    \label{fig:path_alpha}
\end{figure}

\section{\edit{Robustness to the Baseline Choice}}
\label{appx:baseline_choice}
\edit{DiffIG uses a fixed black baseline by default, which (as noted in our Limitations) can under-attribute dark image regions. Here we show that DiffIG is robust to this choice and that a simple remedy mitigates the dark-region issue. We re-run DiffIG with black, white, and gray baseline inputs, and with a black--white mean (BW-mean) attribution~\cite{kapishnikov2019xrai}, $\mathrm{Attr}_{\mathrm{BW}}(x) = \tfrac{1}{2}\bigl(\mathrm{Attr}^{\mathrm{black}}(x) + \mathrm{Attr}^{\mathrm{white}}(x)\bigr)$, a convex combination of two valid DiffIG attributions for the same prediction target that preserves the per-baseline axiomatic properties by construction. No retraining or hyperparameter changes are made; only the baseline choice changes.}

\edit{As shown in Table~\ref{tab:baseline_choice}, DiffIG remains above every path-based competitor (\cf~Table~\ref{tab:quant}) for all six (dataset, model) pairs under all baselines. Averaging the black and white attributions covers each other's failure mode: BW-mean exceeds the best \emph{single} baseline in five of six settings (up to $+0.115$ DiffID), directly addressing the dark-region failure noted in our Conclusion. The lone regression (on Oxford-IIIT Pet with VGG16, $-0.008$) is driven by Deletion alone: its Insertion AUC in fact \emph{improves} from $0.758$ to $0.784$ while Deletion AUC rises from $0.100$ to $0.133$.}

\begin{table}[t!]
    \caption{\edit{\textbf{The DiffID of DiffIG's under alternative  baselines.} BW-mean denotes the mean of the black- and white-baseline attributions; \textbf{bold} marks the best value in each row; $\Delta$ is BW-mean minus the best single baseline.}}
    
    \label{tab:baseline_choice}
    \centering
    \setlength{\tabcolsep}{4pt}
    \begin{tabular}{ll rrrr r}
        \toprule
        \textbf{Dataset} & \textbf{Model} & \textbf{Black} & \textbf{White} & \textbf{gray} & \textbf{BW-mean} & $\boldsymbol{\Delta}$ \\
        \midrule
        Mini-ImageNet & ResNet18    & 0.309 & 0.131 & 0.269 & \textbf{0.311} & $+0.002$ \\
        Mini-ImageNet & VGG16       & 0.410 & 0.266 & 0.322 & \textbf{0.468} & $+0.059$ \\
        Mini-ImageNet & \edit{Inception V1} & 0.353 & 0.269 & 0.342 & \textbf{0.467} & $+0.115$ \\
        Oxford-IIIT Pet     & ResNet18    & 0.492 & 0.345 & 0.457 & \textbf{0.560} & $+0.068$ \\
        Oxford-IIIT Pet     & VGG16       & \textbf{0.659} & 0.517 & 0.553 & 0.651 & $-0.008$ \\
        Oxford-IIIT Pet     & \edit{Inception V1} & 0.484 & 0.366 & 0.528 & \textbf{0.613} & $+0.085$ \\
        \bottomrule
    \end{tabular}
\end{table}

\section{\edit{Validity Under the Latent--Pixel Domain Gap}}
\label{appx:latent_pixel}
\edit{DiffIG trains and applies its guidance predictor on latent-space paths, while the final faithfulness is evaluated in pixel space (Section~``Validity of Guided Generation''). To probe whether this domain gap is harmful, we construct a controlled 3D-helix toy in which a latent DDPM is guided exactly as in DiffIG, but the guidance target is read in pixel space through the decoder. As shown in Figure~\ref{fig:latent_sweep}, increasing the guidance scale $\omega \in \{0, 5, 10, 20\}$ steadily optimizes the pixel-space objective (target-side fraction $0.53 \to 1.00$) while the mean distance to the data manifold stays at $\approx 10^{-3}$. Guidance therefore optimizes the pixel-space objective without leaving the manifold, indicating the latent--pixel gap is benign in practice.}

\begin{figure}[t!]
    \centering
    \includegraphics[width=\linewidth]{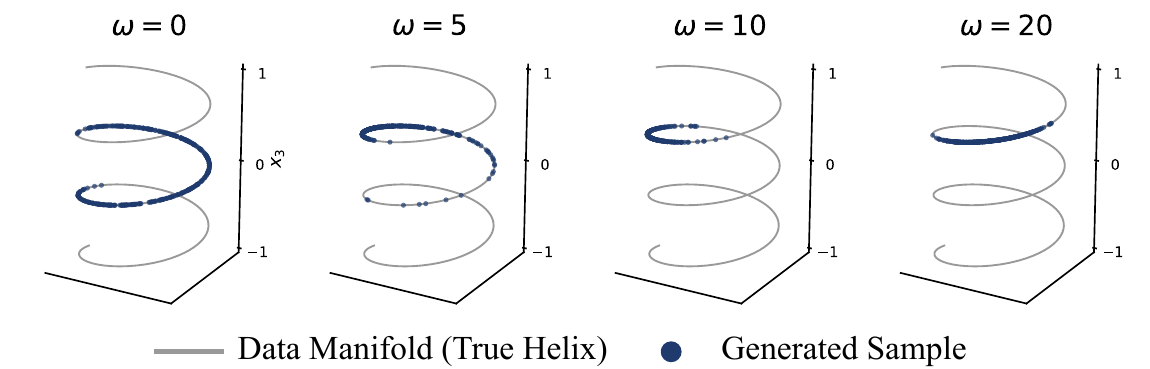}
    \caption{\edit{\textbf{Latent guidance under a pixel-space target (3D-helix toy).} As $\omega$ increases, the target-side fraction rises from $0.53$ to $1.00$ while the mean manifold distance remains $\approx 10^{-3}$, showing guidance optimizes the pixel objective without leaving the manifold.}}
    \label{fig:latent_sweep}
\end{figure}

\section{\edit{Comparison with Smoothing Attribution Methods}}
\label{appx:smoothing}

\edit{To situate DiffIG within the broader attribution literature, we test whether DiffIG's gains can be reproduced by simply smoothing IG. All experiments follow our protocol on Oxford-IIIT Pet with ResNet18.}

\edit{Reference~\cite{smilkov2017smoothgrad} and related work show that low-pass filtering can improve faithfulness scores for gradient-based maps. We therefore ask whether IG ${+}$ smoothing closes the gap. Table~\ref{tab:smoothing} shows that post-hoc Gaussian smoothing of IG monotonically reduces DiffID. SmoothGrad (SG)~\cite{smilkov2017smoothgrad} and SG-IG improve over plain IG but remain well below DiffIG$_{N=10}$. Crucially, smoothing leaves the straight-line integration path unchanged, indicating that DiffIG's gain comes from controllable path generation rather than low-pass filtering.}

\begin{table}[t!]
    \centering
    \footnotesize
    \caption{\edit{\textbf{Smoothing baselines on Oxford-IIIT Pet/ResNet18.} Smoothing ablation using 100 validation samples and SmoothGrad with $M{=}20$. \textbf{Bold} marks the best value in each block.}}
    \label{tab:smoothing}
    \setlength{\tabcolsep}{4pt}
    \begin{tabular}{lccc}
        \toprule
        \textbf{Method} & \textbf{DiffID ($\uparrow$)} & \textbf{Ins ($\uparrow$)} & \textbf{Del ($\downarrow$)} \\
        \midrule
        IG~\cite{sundararajan2017axiomatic}             & 0.363 & 0.536 & 0.173 \\
        IG ${+}$ Gauss.\ $\sigma{\in}\{1,2,4\}$         & 0.348/0.313/0.266 & 0.513/0.469/0.412 & 0.165/0.156/0.146 \\
        SG / SG-IG~\cite{smilkov2017smoothgrad}          & 0.447 / 0.390 & 0.582 / 0.552 & \textbf{0.135} / 0.162 \\
        DiffIG$_{N=1}$ / $_{N=10}$                       & 0.384 / \textbf{0.509} & 0.534 / \textbf{0.660} & 0.151 / 0.151 \\
        \bottomrule
    \end{tabular}
\end{table}

\section{\edit{Faithfulness Under an Independent Metric}}
\label{appx:quantus}
\edit{A possible concern is that DiffIG's advantage stems from optimizing the same Insertion/Deletion metric used for evaluation. It does not: the guidance predictor is a proxy regressor trained on noisy latent paths, so DiffIG never optimizes the pixel-space benchmark directly. To verify this on an independent metric, we evaluate the \emph{Faithfulness Correlation} (using absolute-value attributions) from the Quantus toolkit~\cite{hedstrom2023quantus} on Oxford-IIIT Pet/ResNet18 ($n{=}370$). DiffIG attains $-0.0025$, ahead of AGI ($-0.0042$) and GIG ($-0.0093$), confirming that the improvement transfers to a metric DiffIG was not tuned against.}

\section{\edit{Verifying the Meerkat Attribution}}
\label{appx:meerkat}
\edit{In the main paper (Figure~\ref{fig:qual}), DiffIG attributes importance to a third meerkat on the far right that other methods miss. Because attribution is model-specific (a region should be salient if and only if perturbing it changes the classifier output), we verify that this region carries genuine classifier evidence rather than DiffIG over-amplification. Using the original sample, \emph{Insertion} of the far-right ROI onto a mean-masked baseline yields a positive target-class change in both logit ($+0.346$) and probability ($+0.0094$). Conversely, \emph{Deletion} of the same ROI lowers the target logit by $1.3048/1.4003/1.3416$ under black/mean/blur in-painting, respectively. Both tests confirm the far-right region is genuine model evidence (Figure~\ref{fig:meerkat_rois}).}

\begin{figure}[t!]
    \centering
    \includegraphics[width=0.4\linewidth]{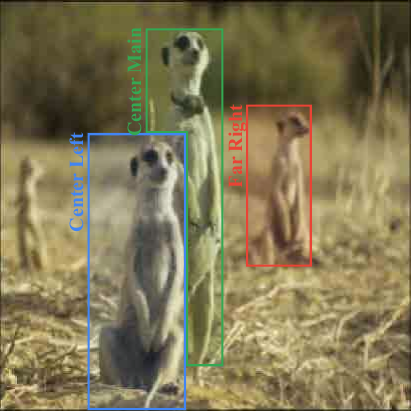}
    \caption{\edit{\textbf{Region-of-interest verification on the meerkat example.} Inserting the far-right ROI raises the target-class logit/probability, and deleting it lowers the target logit across black/mean/blur in-painting, confirming the region is genuine classifier evidence.}}
    \label{fig:meerkat_rois}
\end{figure}

\section{\edit{Per-Step Attribution Across Denoising Steps}}
\label{appx:denoising_steps}
\edit{To visualize how the attribution forms during sampling, Figure~\ref{fig:denoising_progression} shows the per-step DiffIG attribution at relative denoising progress $t/T \in \{0, 0.2, 0.4, 0.6, 0.8, 1.0\}$ for an Oxford-IIIT Pet/ResNet18 sample ($N{=}10$, median-aggregated). The maps are noise-like for $t/T \le 0.4$ and concentrate on the target object from $t/T \approx 0.6$ onward; the trajectory mirrors the dataset-level DiffID curve, which rises from $0.01$ to $0.46$ over the same interval.}

\begin{figure}[t!]
    \centering
    \includegraphics[width=\linewidth]{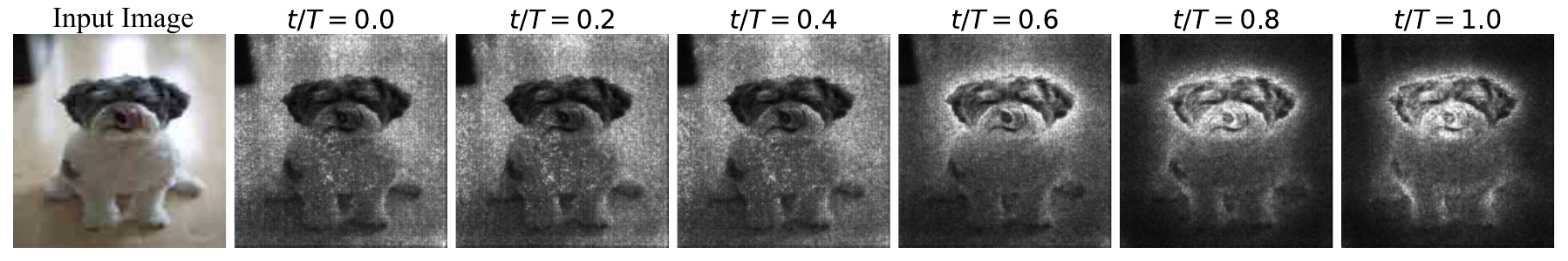}
    \caption{\edit{\textbf{Per-step DiffIG attribution across denoising progress} $t/T \in \{0, 0.2, 0.4, 0.6, 0.8, 1.0\}$ (Oxford-IIIT Pet/ResNet18, $N{=}10$ median-aggregated). Maps are noise-like for $t/T \le 0.4$ and concentrate on the target object from $t/T \approx 0.6$ onward.}}
    \label{fig:denoising_progression}
\end{figure}

\section{\edit{Failure Case Analysis}}
\label{appx:failure_cases}
\edit{We examine cases where DiffIG underperforms the strongest path-based baselines (GIG, AGI). Most such cases are not true failures but favorable joint-objective trade-offs, where DiffIG concedes a little on one objective (faithfulness or complexity) to gain substantially on the other. For instance, GIG can attain lower complexity, but by concentrating attribution mass in scattered, hot-pixel components, whereas DiffIG forms more coherent foci that recover a large DiffID margin (Figure~\ref{fig:failure_cases}). 
Genuine failures, where DiffIG is worse than both baselines on both metrics, are comparatively rare and typically arise in dark-object cases, suggesting that such inputs may be more challenging for controllable path generation.}

\begin{figure}[t!]
    \centering
    \includegraphics[width=0.55\linewidth]{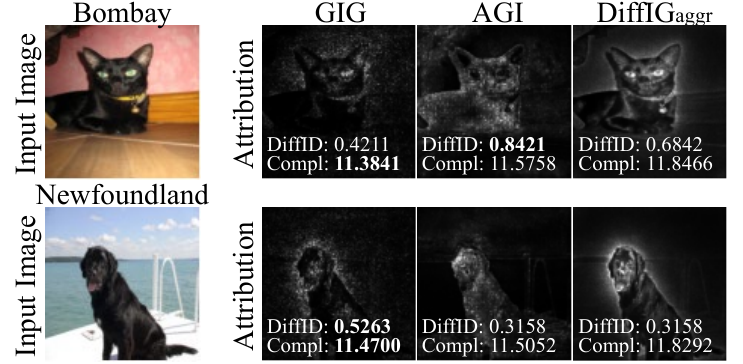}
    \caption{\edit{\textbf{Representative trade-off/failure cases (Oxford-IIIT Pet/ResNet18).} For the \emph{Bombay} and \emph{Newfoundland} examples, GIG attains lower complexity but through fragmented hot-pixel attribution (many small components), whereas DiffIG\textsubscript{aggr} recovers a large DiffID margin with more coherent foci. DiffID is higher-better and complexity is lower-better.}}
    \label{fig:failure_cases}
\end{figure}

\section{Hyperparameter Sensitivity Analysis}
\label{appx:optimal_hyperparameters}

To assess the robustness of DiffIG and understand the influence of its key design choices, we conducted a structured sensitivity analysis across \edit{six} experimental settings (Oxford-IIIT Pet vs. Mini-ImageNet; \edit{VGG16, ResNet18, and Inception V1}). Rather than aiming to exhaustively tune for the single best-performing configuration, our goal was to examine how stable DiffIG remains under different hyperparameter combinations and to identify parameters that have consistent, interpretable effects on performance. We analyze three hyperparameter groups:
\begin{itemize}
    \item \textbf{Aggregation method}: Strategies for combining $N>1$ sampled paths, including \textsf{Best}, \textsf{SPI-P}, \textsf{Median}, \textsf{Mean}, and variance-weighted mean (\textsf{V-Mean}).
    \item \textbf{Scaling factors} ($\lambda_{\text{faith}}, \lambda_{\text{comp}}$): Guidance strengths used during conditional sampling.
    \item \textbf{Number of path candidates} ($N$): The number of path candidates to sample (where $N$=1 corresponds to a single path).
\end{itemize}

We report performance using the standard perturbation metrics (DiffID, Insertion, and Deletion), allowing us to assess both sensitivity and consistency across settings.

\subsection{Ablation Study on the Number of Path Candidates}

Across all datasets, models, and aggregation methods, we observe a highly stable and monotonic trend: performance improves as the number of sampled candidates increases. Figure~\ref{fig:n_candidates_ablation} shows that $N{=}1$ consistently underperforms, while increasing the number of sampled paths to $N{=}10$ yields a substantial improvement. The gains largely saturate at $N{=}50$, which achieves the most reliable performance across all settings. This pattern indicates that DiffIG benefits from exploring multiple plausible attribution paths, and it confirms that the method is robust with respect to $N$ as long as $N$ is moderately large. 

\begin{figure}[t!]
    \centering
    \includegraphics[width=\linewidth]{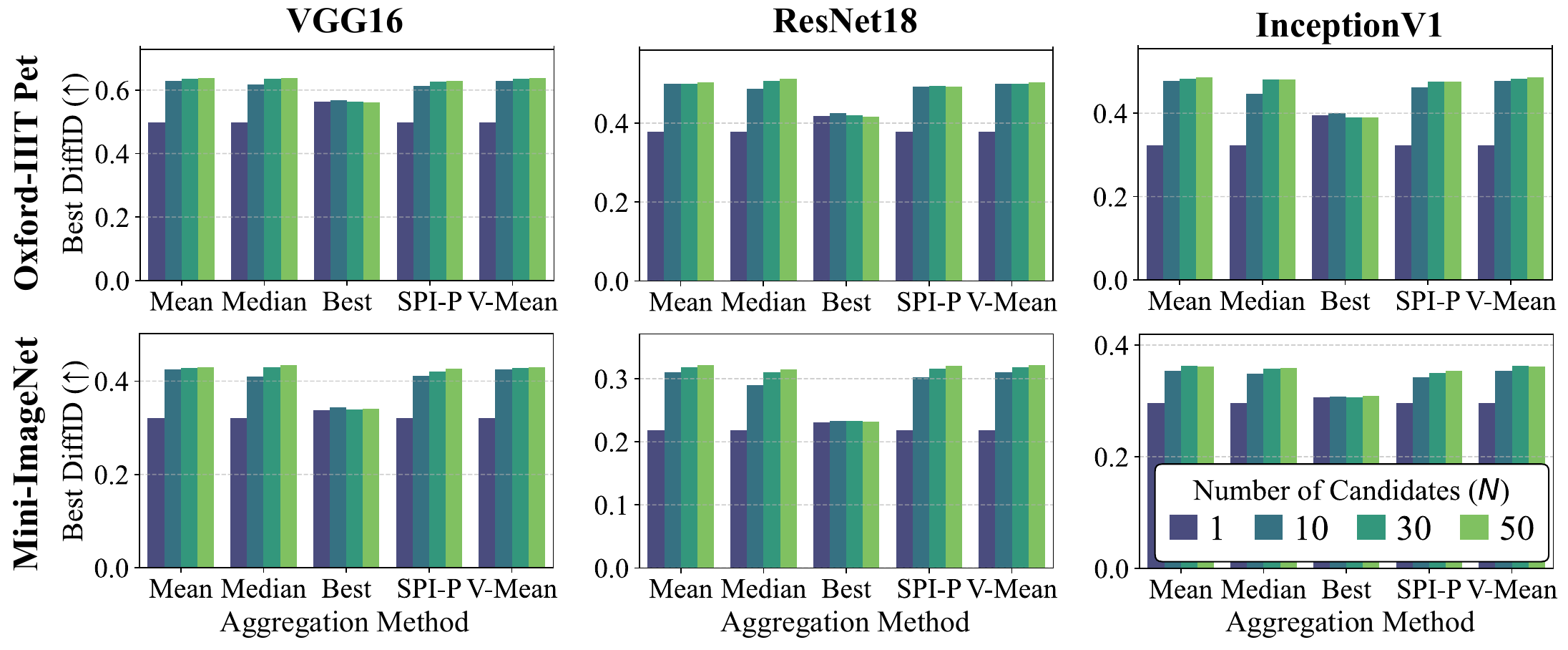}
    \caption{\textbf{DiffID score by number of path candidates ($N$) and aggregation method.} Higher scores are better. The x-axis shows the aggregation method, while bar colors denote $N$. Error bars represent the standard deviation computed across all combinations of scaling factors $\lambda_{\text{faith}}$ and $\lambda_{\text{comp}}$.}
    \label{fig:n_candidates_ablation}
\end{figure}

\begin{figure}[t!]
    \centering
    \includegraphics[width=\linewidth]{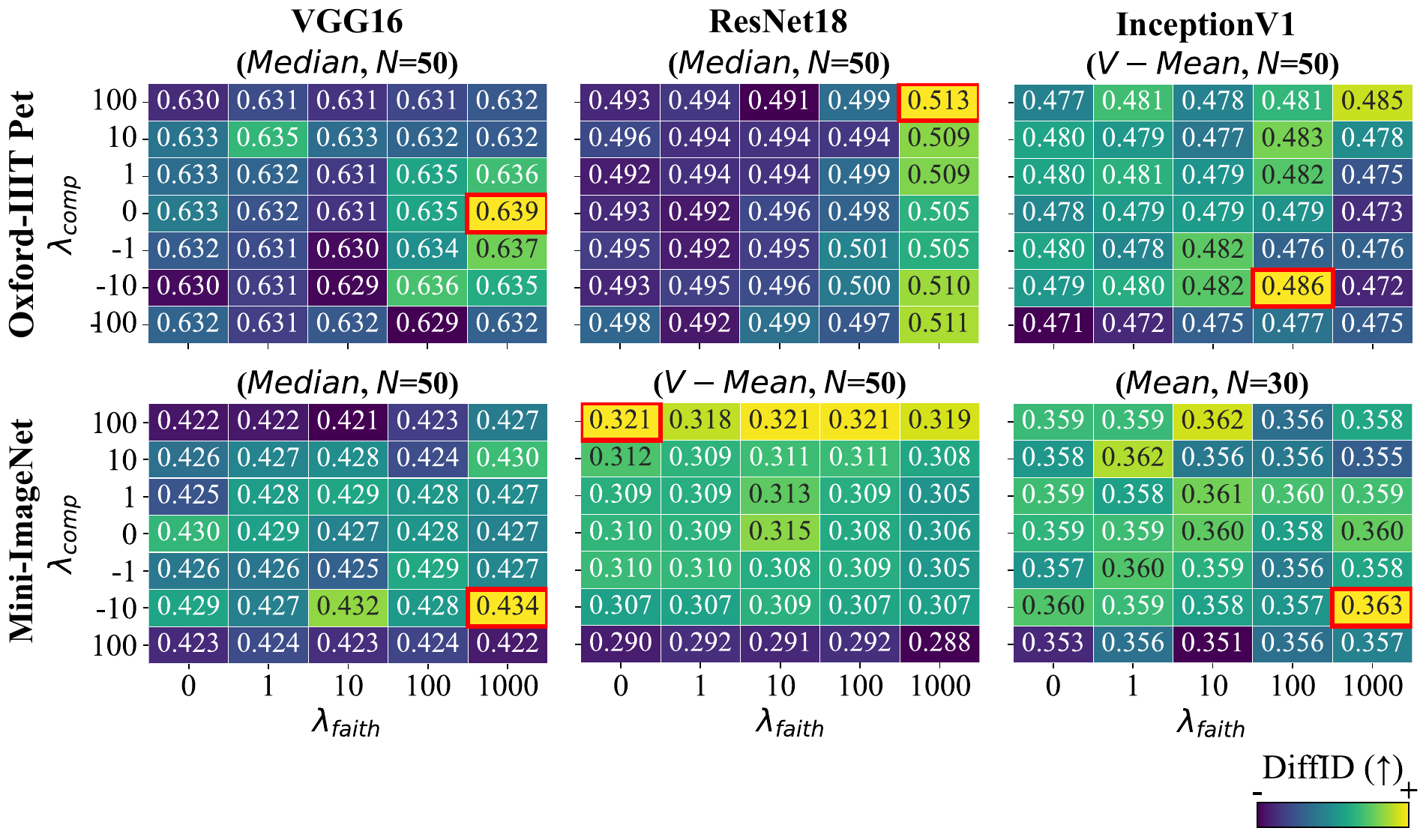}
    \caption{\textbf{Scaling factors $\lambda_{\text{faith}}$ and $\lambda_{\text{comp}}$ analysis.} The heatmap shows the DiffID score (higher is better). The aggregation method and the number of path candidates were fixed to their optimal values for each dataset and model.}
    \label{fig:w_cg_heatmap_analysis}
\end{figure}

\subsection{Ablation Study on Aggregation Method}

The choice of aggregation method proved to be highly influential, creating a clear separation in performance across all experimental setups. This is visually evident in Figures \ref{fig:ins_del_tradeoff} and \ref{fig:ins_faith_tradeoff}, which both show two distinct clusters.

The \edit{Best-of-$N$} (\textsf{Best}) method (blue dots) consistently forms an inferior cluster, characterized by high Deletion scores (undesirable) and low Insertion scores (undesirable) as seen in Figure \ref{fig:ins_del_tradeoff}. Similarly, Figure \ref{fig:ins_faith_tradeoff} shows it is isolated in a low-DiffID, low-Insertion region.

Conversely, the other four methods—\textsf{SPI-P}, \textsf{Median}, variance-weighted mean(\textsf{V-Mean}), and \textsf{Mean}—form a superior cluster of high-performing configurations. Figure~\ref{fig:ins_del_tradeoff} demonstrates that these four methods collectively define the optimal Pareto Frontier (dashed line), representing the best possible trade-off between minimizing Deletion and maximizing Insertion. \edit{Notably, despite its name, the \textsf{Best} method does not lie on this frontier, confirming that the Best-of-$N$ selection strategy underperforms the other four aggregation methods.}

Figure \ref{fig:ins_faith_tradeoff} reinforces this finding, showing these four methods tightly grouped in the desirable top-right quadrant (high DiffID, high Insertion). The zoomed-in insets further confirm that while there are minor differences—for instance, \textsf{Median} (yellow) often slightly lags \textsf{Mean} (cyan) and \textsf{V-Mean} (green)—all four methods represent a robust and desirable trade-off, far superior to the \textsf{Best} method.

\begin{figure}[t!]
    \centering
    \includegraphics[width=\linewidth]{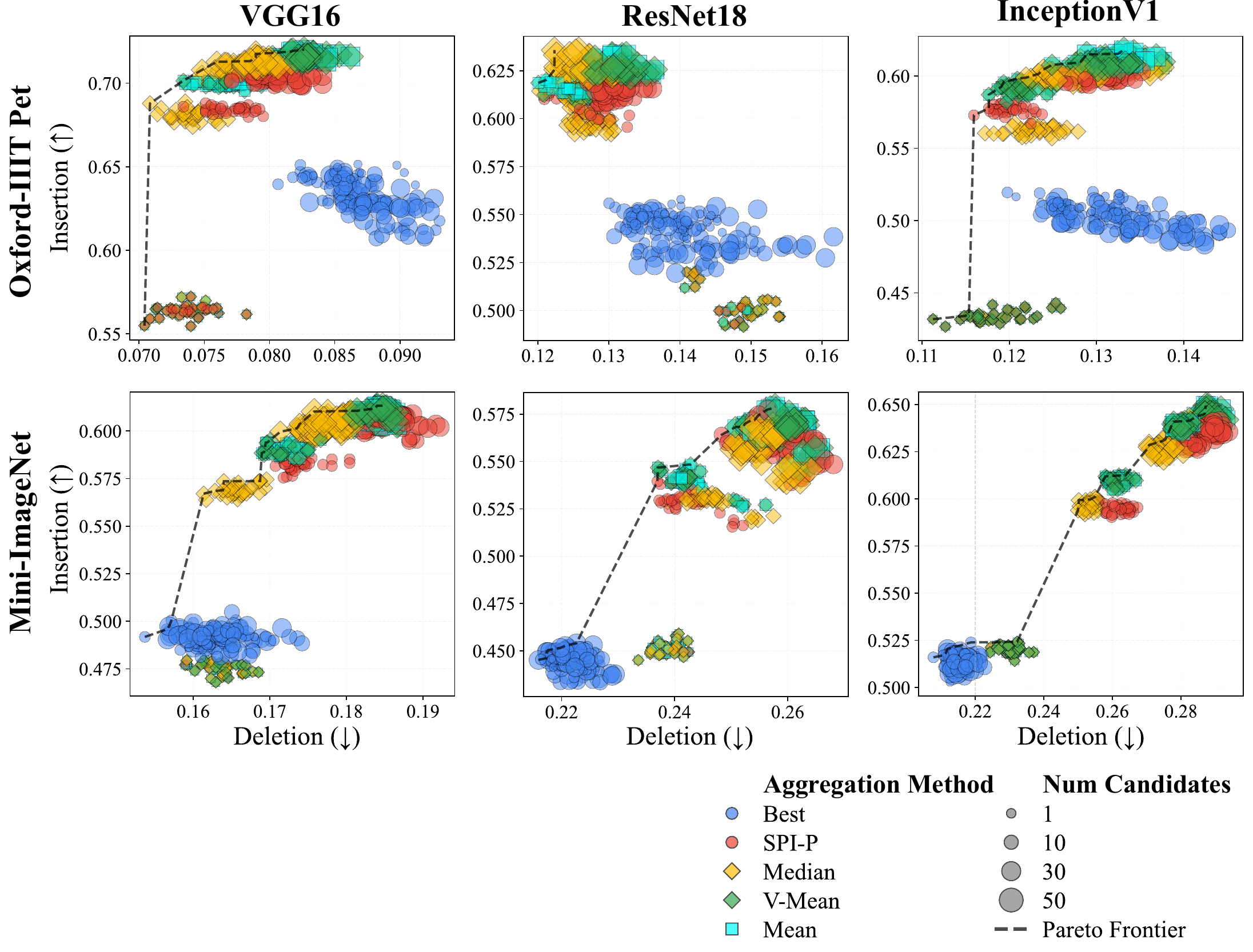}
    \caption{\textbf{Hyperparameter analysis for aggregation method and number of path candidates on the Deletion ($\downarrow$) vs. Insertion ($\uparrow$) metrics.} The plot clearly isolates the sub-optimal {\color{RoyalBlue}\textsf{Best}} method (blue). The four other methods ({\color{red}\textsf{SPI-P}}, {\color{orange}\textsf{Median}}, {\color{Green}\textsf{V-Mean}}, {\color{Cerulean}\textsf{Mean}}) form the superior performance cluster, with their optimal trade-offs illustrated by the Pareto Frontier (dashed line).}
    \label{fig:ins_del_tradeoff}
\end{figure}

\begin{figure}[t!]
    \centering
    \includegraphics[width=\linewidth]{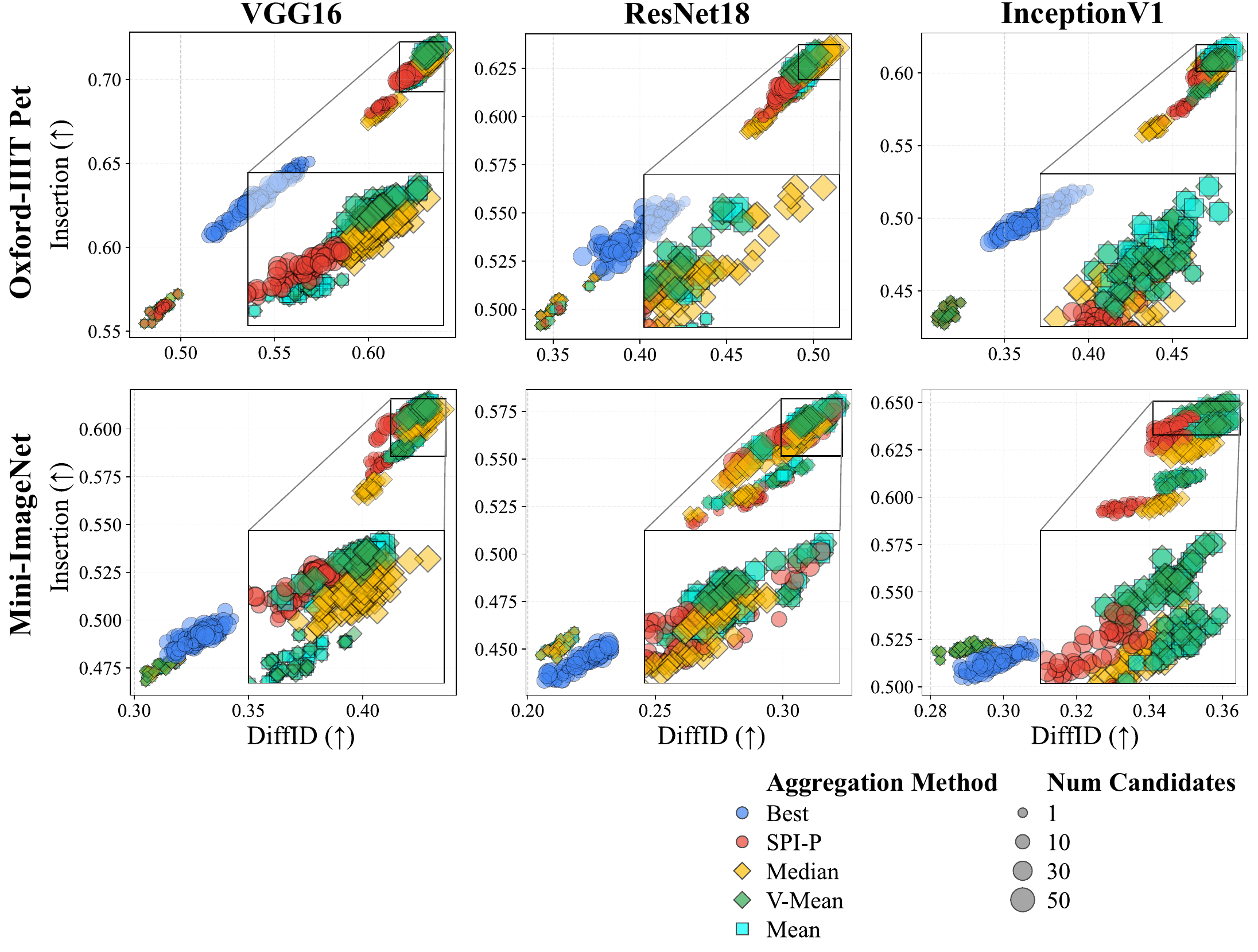}
    \caption{\textbf{Hyperparameter analysis for aggregation method and number of path candidates using DiffID ($\uparrow$) and Insertion ($\uparrow$) metrics.} The {\color{RoyalBlue}\textsf{Best}} method (blue) is identified as sub-optimal. The inset provides a magnified view of the high-performing cluster (formed by the other four methods) to analyze the optimal settings and the impact of number of path candidates.}
    \label{fig:ins_faith_tradeoff}
\end{figure}

\subsection{Ablation Study on Scaling Factors}

The weighting parameters $\lambda_{\text{faith}}$ and $\lambda_{\text{comp}}$ were observed to be sensitive to the specific dataset and model architecture, as visualized in the heatmaps in Figure~\ref{fig:w_cg_heatmap_analysis}. The optimal settings, listed below Table~\ref{tab:rank1_summary}, vary significantly across experiments.

While the optimal settings for $\lambda_{\text{faith}}$ and $\lambda_{\text{comp}}$ vary significantly and the actual performance gap is relatively modest, the heatmaps nonetheless reveal a clear and important general trend: applying guidance is crucial for performance. In all four plots, the bottom-left corner, which corresponds to no guidance ($\lambda_{\text{faith}}=0, \lambda_{\text{comp}}=0$), consistently results in the lowest faithfulness scores (dark purple).

Conversely, the highest-performing regions (bright yellow) are always found where at least one, or both, of the scaling factors has a high value. For example, performance on Oxford-IIIT Pet is maximized with a high $\lambda_{\text{faith}}$ (right side of the plots), while VGG16 on Mini-ImageNet requires a high $\lambda_{\text{comp}}$ (top of the plot). This strongly suggests that while the optimal balance between faithfulness and complexity guidance is model-dependent, applying significant guidance weights is generally beneficial.

\subsection{Optimal Hyperparameter Settings}
Table \ref{tab:rank1_summary} summarizes the parameter settings that achieved the best result for the DiffID metric in each experiment. \edit{The DiffID values reported here are the peak scores identified during this hyperparameter search; the main results in Table~\ref{tab:quant} instead use a single fixed configuration---a black baseline, \textsf{Median} aggregation, and $N{=}30$---across all (dataset, model) pairs for consistency and fair comparison.}
In conclusion, the hyperparameter search shows that increasing the number of candidates ($N$) improves performance, saturating once $N$ is moderately large; that \textsf{Median}, \textsf{Mean}, and \textsf{V-Mean} are all effective aggregation choices; and that $\lambda_{\text{faith}}$ and $\lambda_{\text{comp}}$ are sensitive and require per-setup tuning.

\begin{table}[t!]
    \caption{\textbf{Summary of best hyperparameter configurations} based on the DiffID score. (\textbf{Aggr.}: Aggregation method, \textbf{N}: The Number of Path Candidates)}
    \setlength{\tabcolsep}{8pt}
    \centering
    \begin{tabular}{llclrrc}
    \toprule
    \textbf{Model} & \textbf{Dataset} & \textbf{DiffID} & \textbf{Aggr.} & \textbf{$\lambda_{\text{faith}}$} & \textbf{$\lambda_{\text{comp}}$} & \textbf{N} \\
    \midrule
    \multirow{2}{*}{\textbf{VGG16}} 
    & Oxford-IIIT Pet & 0.6385 & \textbf{\textsf{Median}} & \textbf{1000} & 0 & \textbf{50} \\
    & Mini-ImageNet & 0.4344 & \textbf{\textsf{Median}} & \textbf{1000} & \textbf{-10} & \textbf{50} \\
    \cmidrule{1-7}

    \multirow{2}{*}{\textbf{ResNet18}}
    & Oxford-IIIT Pet & 0.5134 & \textbf{\textsf{Median}} & \textbf{1000} & 100 & \textbf{50} \\
    & Mini-ImageNet & 0.3213 & \textsf{V-Mean} & 0 & 100 & \textbf{50} \\
    \cmidrule{1-7}
    
    \multirow{2}{*}{\textbf{\edit{Inception V1}}}
    & Oxford-IIIT Pet & 0.4861 & \textsf{V-Mean} & 100 & \textbf{-10} & \textbf{50} \\
    & Mini-ImageNet & 0.3629 & \textsf{Mean} & \textbf{1000} & \textbf{-10} & 30 \\
    \bottomrule
    \end{tabular}
    \label{tab:rank1_summary}
\end{table}

\begin{table*}[t!]
\caption{\textbf{Hyperparameters for DiffIG.}}
\label{tab:hyperparams_diffig}
\centering
\begin{tabular}{@{}llc@{}} %
\toprule
\textbf{Component} & \textbf{Hyperparameter} & \textbf{Value} \\ %
\midrule

\multicolumn{3}{l}{\textit{Latent Encoder–Decoder ($\mathcal{E}_{\text{enc}}, \mathcal{D}_{\text{dec}}$)}} \\ %
& Architecture & $\beta$-VAE \\ %
& Latent Dimension & 4096 \\ %

\midrule
\multicolumn{3}{l}{\textit{Unconditional Diffusion Model ($p_\theta(\gamma)$)}} \\ %
& Network Backbone & DiT1D \\ %
& Learning Rate & $2 \times 10^{-4}$ \\ %
& Weight Decay & $1 \times 10^{-5}$ \\ %
& Batch Size & 64 \\ %
% & Noise Scheduler & Linear \\
& Diffusion Steps ($M$) & 100 \\ %
& Training Steps & 1,000,000 \\ %
% & Loss Function & MSE ($\epsilon$-prediction) \\

\midrule
\multicolumn{3}{l}{\textit{Guidance Network ($\mathcal{J}_{\phi_1}, \mathcal{J}_{\phi_2}$)}} \\ %
& Network Backbone & DiT1D \\ %
& Learning Rate & $2 \times 10^{-4}$ \\ %
& Weight Decay & $1 \times 10^{-5}$ \\ %
& Batch Size & 64 \\ %
& Training Steps & 1,000,000 \\ %
% & Objective & MSE Regression \\

\midrule
\multicolumn{3}{l}{\textit{Sampling and Guidance Parameters}} \\ %
& Solver & DDPM \\ %
& Sampling Steps & 100 \\ %
& Faithfulness Weight ($\lambda_{\text{faith}}$) & $\{0.0, 1.0, 10.0, 100.0, 1000.0\}$ \\ %
& Complexity Weight ($\lambda_{\text{comp}}$) & $\{-100.0, -10.0, -1.0, 0.0, 1.0, 10.0, 100.0\}$ \\ %
& Number of Paths ($N$) & $\{1, 10, 30, 50\}$ \\ %
& Aggregation Operator & $\{\text{Mean}, \text{Median}, \text{V-Mean}, \text{SPI-P}\}$ \\ %

\midrule
\multicolumn{3}{l}{\textit{Synthetic Path Generation (SBP)}} \\ %
& Base Distribution ($H$) & Uniform $(0,1)$ \\ %
& Concentration Parameter ($\alpha$) & $\in [1.0, 20.0]$ \\ %

\bottomrule
\end{tabular}%
\end{table*}

\section{Additional Qualitative Results}\label{appx:additional_qual}

In Figure~\ref{fig:additional_qual1}, \ref{fig:additional_qual2}, \ref{fig:additional_qual3}, \ref{fig:additional_qual4}, \ref{fig:additional_qual5} and Figure~\ref{fig:additional_qual6}, we present additional qualitative attribution results. These figures provide further comparisons of our DiffIG variants (\textsf{N=1}, \textsf{best}, and \textsf{aggr}) against the baseline path-based methods on validation samples from the Mini-ImageNet and Oxford-IIIT Pet datasets.

\begin{figure*}
    \centering
    \includegraphics[width=0.95\linewidth]{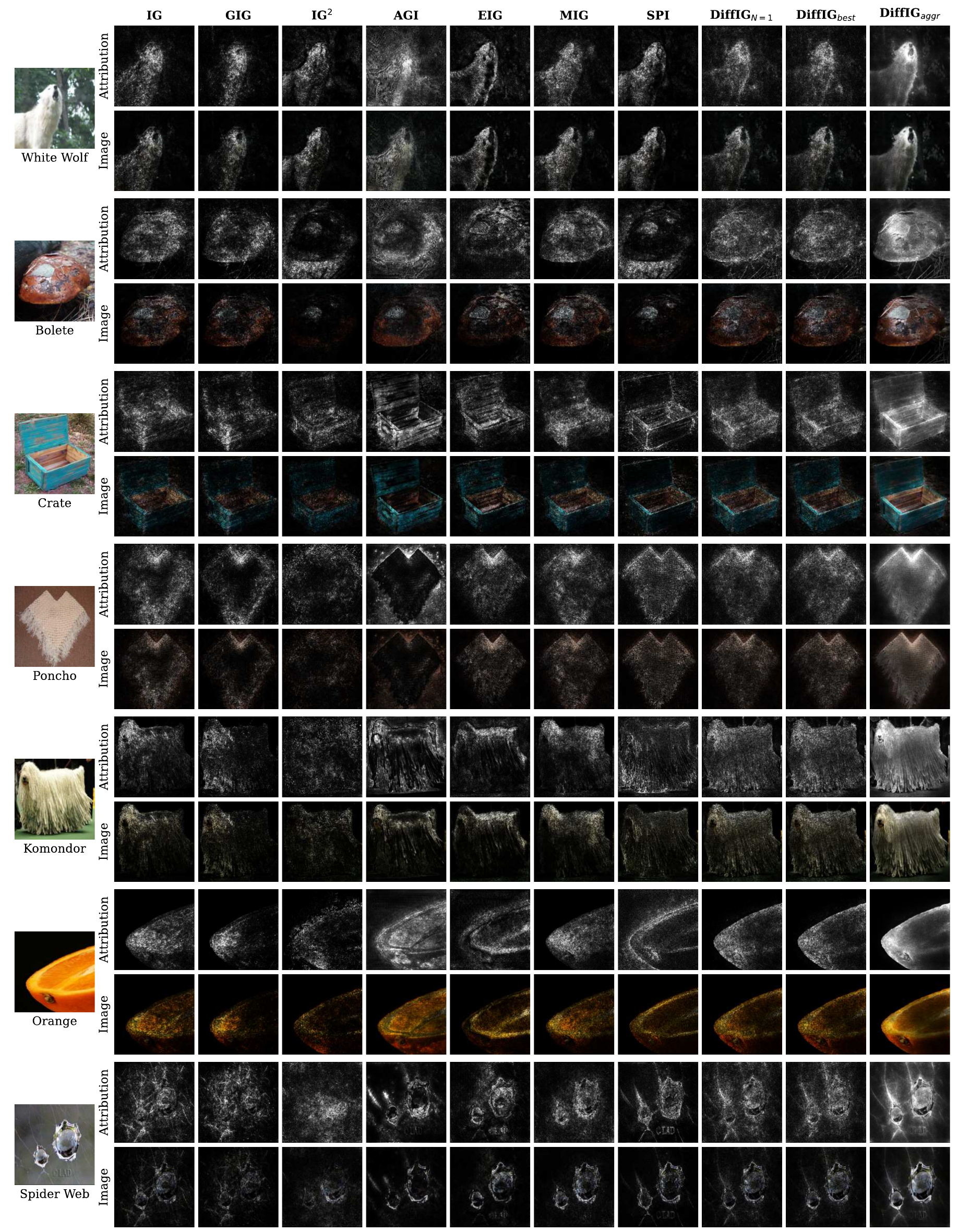}
    \caption{\textbf{Attribution examples of the VGG16 model on the Mini-ImageNet validation set.} Three DiffIG variants are compared: a single path candidate (\textsf{\textit{N}=1}), the \edit{Best-of-$N$} selection strategy (\textsf{best}), and the median aggregation method (\textsf{aggr}). The \textit{Attribution} rows show the attribution maps (min-max normalized after 99th-percentile clipping). The \textit{Image} rows show them overlaid on the input images.}
    \label{fig:additional_qual1}
\end{figure*}

\begin{figure*}
    \centering
    \includegraphics[width=0.95\linewidth]{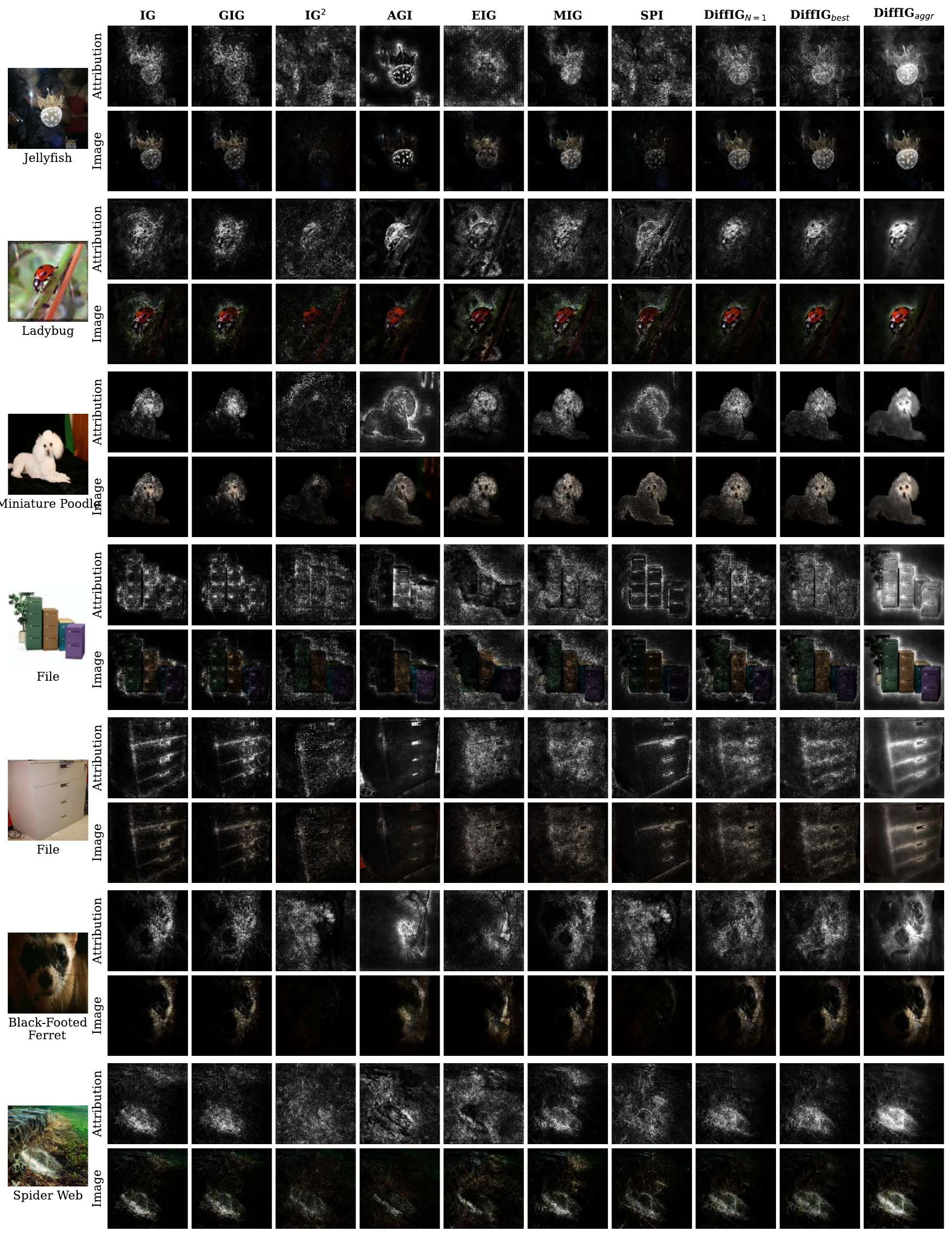}
    \caption{\textbf{Attribution examples of the VGG16 model on the Mini-ImageNet validation set.} Three DiffIG variants are compared: a single path candidate (\textsf{\textit{N}=1}), the \edit{Best-of-$N$} selection strategy (\textsf{best}), and the median aggregation method (\textsf{aggr}). The \textit{Attribution} rows show the attribution maps (min-max normalized after 99th-percentile clipping). The \textit{Image} rows show them overlaid on the input images.}
    \label{fig:additional_qual2}
\end{figure*}

\begin{figure*}
    \centering
    \includegraphics[width=0.95\linewidth]{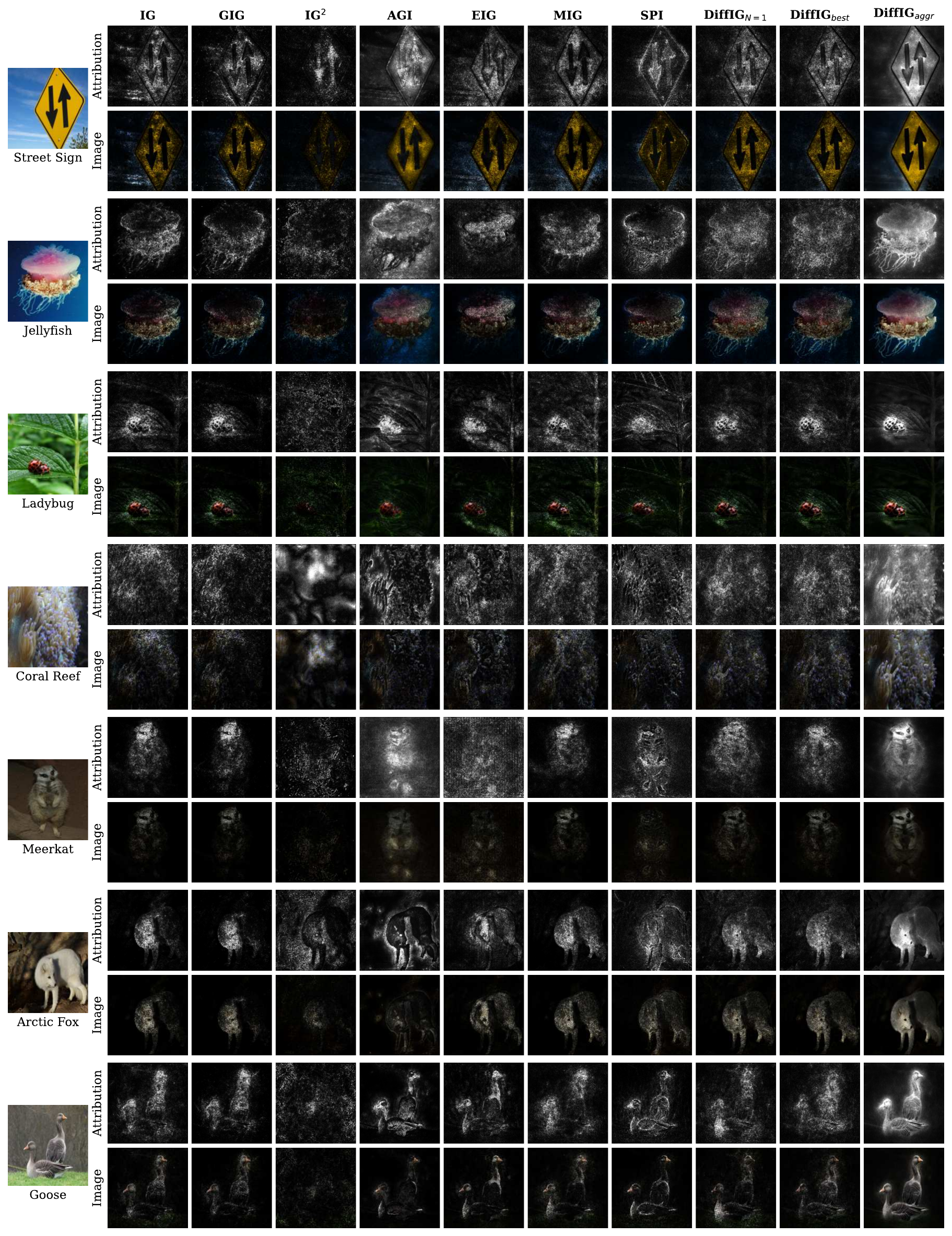}
    \caption{\textbf{Attribution examples of the VGG16 model on the \edit{Mini-ImageNet} \edit{validation} set.} Three DiffIG variants are compared: a single path candidate (\textsf{\textit{N}=1}), the \edit{Best-of-$N$} selection strategy (\textsf{best}), and the median aggregation method (\textsf{aggr}). The \textit{Attribution} rows show the attribution maps (min-max normalized after 99th-percentile clipping). The \textit{Image} rows show them overlaid on the input images.}
    \label{fig:additional_qual3}
\end{figure*}

\begin{figure*}
    \centering
    \includegraphics[width=0.95\linewidth]{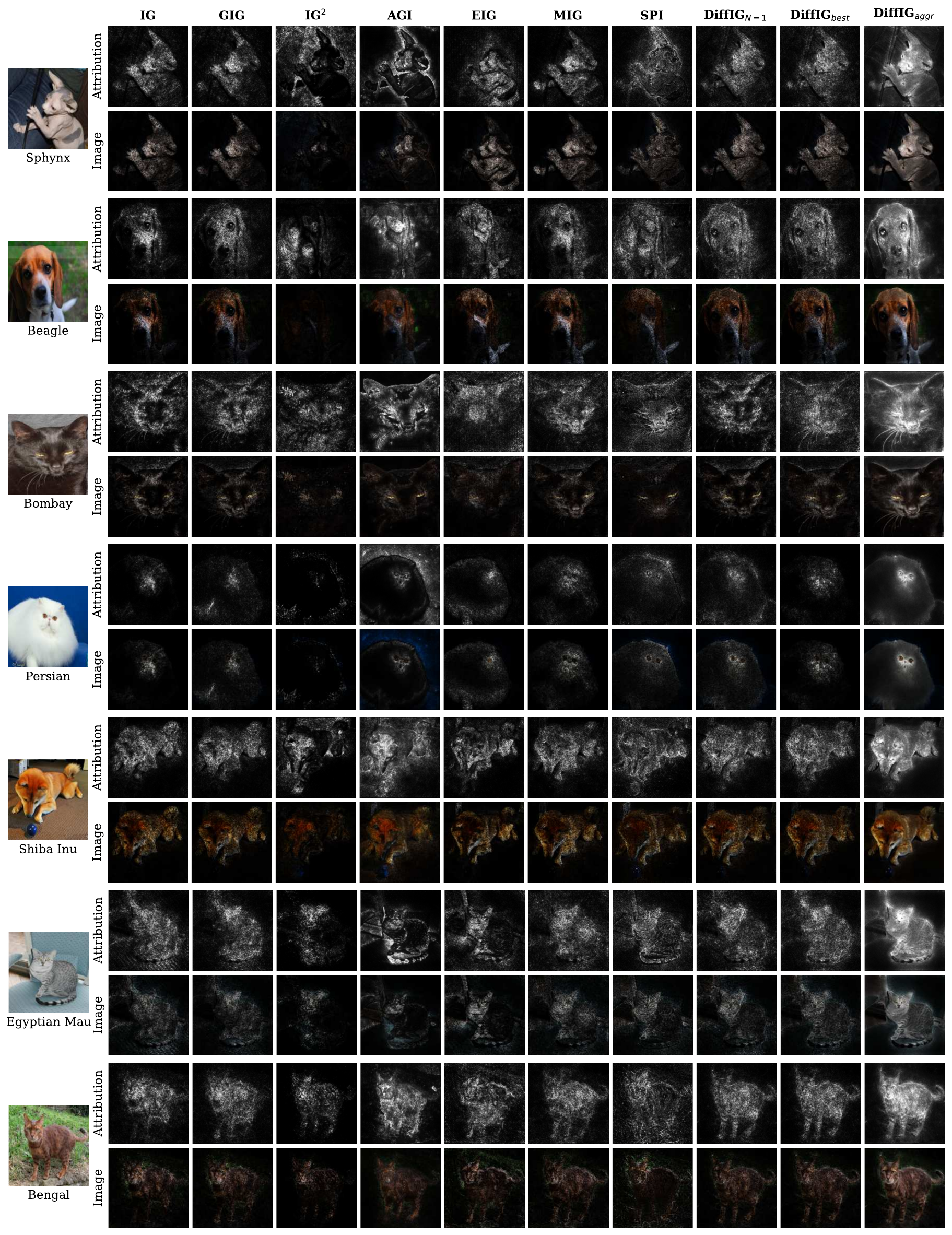}
    \caption{\textbf{Attribution examples of the VGG16 model on the Oxford-IIIT Pet \edit{validation} set.} Three DiffIG variants are compared: a single path candidate (\textsf{\textit{N}=1}), the \edit{Best-of-$N$} selection strategy (\textsf{best}), and the median aggregation method (\textsf{aggr}). The \textit{Attribution} rows show the attribution maps (min-max normalized after 99th-percentile clipping). The \textit{Image} rows show them overlaid on the input images.}
    \label{fig:additional_qual4}
\end{figure*}

\begin{figure*}
    \centering
    \includegraphics[width=0.95\linewidth]{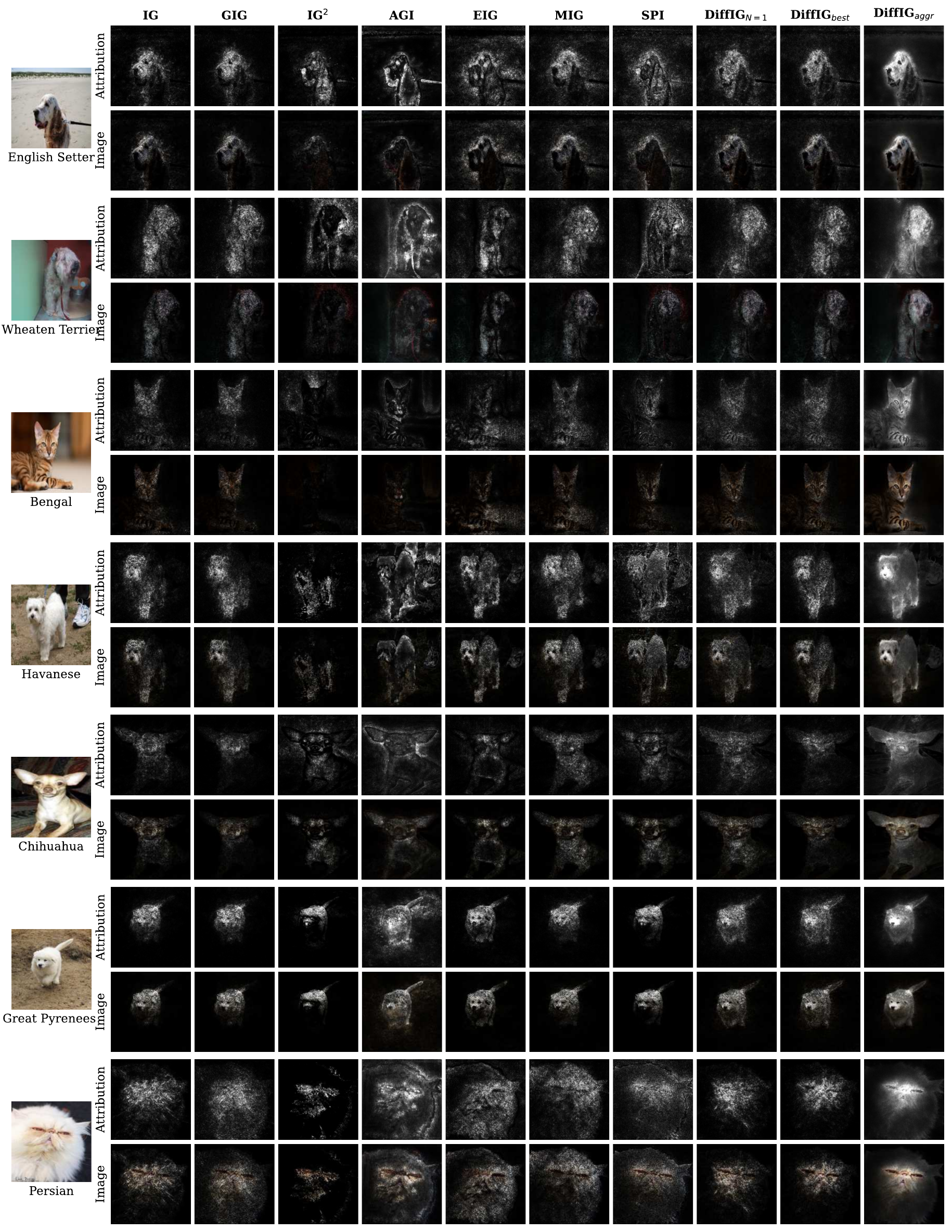}
    \caption{\textbf{Attribution examples of the VGG16 model on the Oxford-IIIT Pet \edit{validation} set.} Three DiffIG variants are compared: a single path candidate (\textsf{\textit{N}=1}), the \edit{Best-of-$N$} selection strategy (\textsf{best}), and the median aggregation method (\textsf{aggr}). The \textit{Attribution} rows show the attribution maps (min-max normalized after 99th-percentile clipping). The \textit{Image} rows show them overlaid on the input images.}
    \label{fig:additional_qual5}
\end{figure*}

\begin{figure*}
    \centering
    \includegraphics[width=0.95\linewidth]{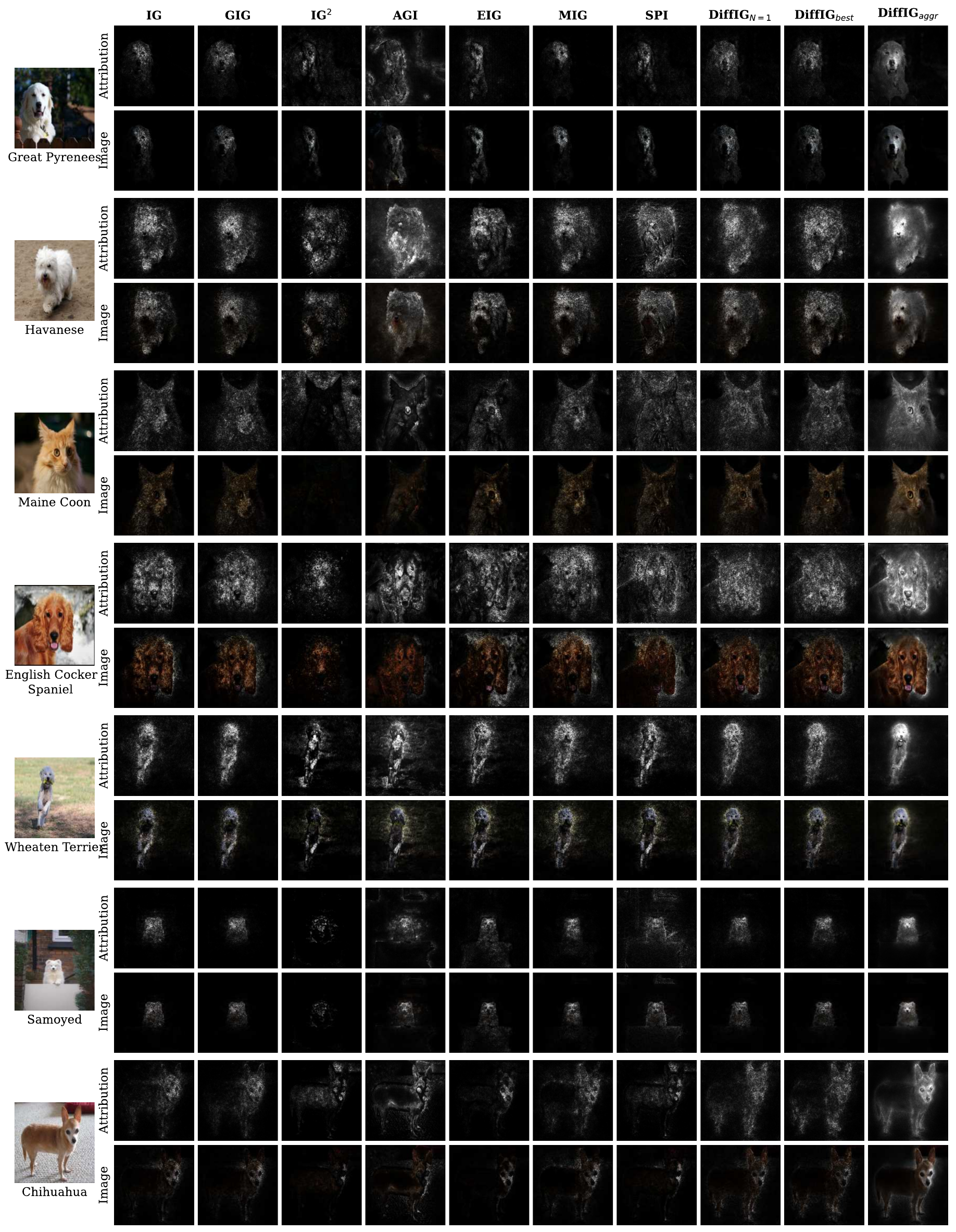}
    \caption{\textbf{Attribution examples of the VGG16 model on the Oxford-IIIT Pet \edit{validation} set.} Three DiffIG variants are compared: a single path candidate (\textsf{\textit{N}=1}), the \edit{Best-of-$N$} selection strategy (\textsf{best}), and the median aggregation method (\textsf{aggr}). The \textit{Attribution} rows show the attribution maps (min-max normalized after 99th-percentile clipping). The \textit{Image} rows show them overlaid on the input images.}
    \label{fig:additional_qual6}
\end{figure*}

\end{document}